\RequirePackage[l2tabu, orthodox]{nag}
\documentclass[12pt]{article}

\usepackage[parfill]{parskip}
\usepackage[margin=1in]{geometry}
\usepackage[defaultlines=3,all]{nowidow}

\usepackage[T1]{fontenc}
\usepackage{microtype}
\usepackage[english]{babel}
\usepackage{libertine}             % Libertine (serif), Biolinium (sans).
\usepackage{titlesec}
\titleformat*{\section}{\Large\bfseries}

\usepackage{booktabs}
\usepackage{float}
\usepackage{graphicx}
\usepackage{makecell}
\usepackage{subcaption}
\usepackage{tablefootnote}
\usepackage{tabularx}
\usepackage{threeparttable}

\usepackage{csquotes}
\usepackage[inline]{enumitem}
\usepackage{pifont} % to write special `ding` symbols % http://ctan.org/pkg/pifont

\usepackage{amsmath}
\usepackage{amssymb}
\usepackage{amsthm}
\usepackage{bbm}
\usepackage{mathtools}
\usepackage{mathrsfs}

\usepackage[dvipsnames,svgnames]{xcolor}

\usepackage[sort]{natbib}

\colorlet{JournalViolet}{BlueViolet!80!black}

\usepackage{hyperref}
\hypersetup{colorlinks,linktoc=all}
\usepackage[all]{hypcap}
\hypersetup{citecolor=JournalViolet}
\hypersetup{linkcolor=red}
\hypersetup{urlcolor=MidnightBlue}

\usepackage[capitalize,noabbrev]{cleveref}
\usepackage{thmtools} % to make cleveref produce correct names for theorem-like blocks.

\usepackage{algorithm}
\usepackage{algorithmic}
\usepackage{multirow}
\usepackage{xspace}
\usepackage{xurl}

\usepackage[affil-it]{authblk}

\usepackage{fancyhdr}
\theoremstyle{plain}
\newtheorem{theorem}{Theorem}[section]
\newtheorem{proposition}[theorem]{Proposition}
\newtheorem{lemma}[theorem]{Lemma}
\newtheorem{corollary}[theorem]{Corollary}

\theoremstyle{definition}
\newtheorem{definition}[theorem]{Definition}
\newtheorem{assumption}[theorem]{Assumption}
\newtheorem{example}[theorem]{Example}

\theoremstyle{remark}
\newtheorem{remark}[theorem]{Remark}

\newcommand{\ARex}{ARex\xspace}
\newcommand{\ARexes}{ARexes\xspace}
\newcommand{\ourprocedure}{Joint-Opt\xspace}
\title{Explanation Design in Strategic Learning:\\ Sufficient Explanations That Induce Non-harmful Responses}

\author[1]{Kiet Q. H. Vo\footnote{Corresponding author: \texttt{huynh.vo@cispa.de}}}
\author[2]{Siu Lun Chau}
\author[3]{Masahiro Kato}
\author[4]{Yixin Wang}
\author[1]{Krikamol Muandet}

\affil[1]{\small{Rational Intelligence Lab, CISPA Helmholtz Center for Information Security, Saarbr\"ucken, Germany}}
\affil[2]{\small{Epistemic Intelligence \& Computation Lab, Nanyang Technological University, Singapore}}
\affil[3]{\small{Mizuho–DL Financial Technology, Co., Ltd., Tokyo, Japan}}
\affil[4]{\small{University of Michigan, Ann Arbor, MI, USA}}

\date{} % keep empty (original ICML paper did not display a date)

\newcommand{\base}[1]{\ddot{#1}}

\newcommand{\basewt}[1]{\ddot{#1}_{t}}
\newcommand{\rec}[1]{\vec{#1}}
\begin{document}
\maketitle

\begin{abstract}
We study the design of explanations in algorithmic decision-making with strategic agents---individuals who may modify their inputs in response to explanations of a decision maker's (DM's) predictive model. While the demand for algorithmic transparency has led much prior work to assume full model disclosure, in practice DMs 
% (e.g., financial institutions) 
typically provide only partial information via explanations, which can cause agents to misinterpret the model and take actions that unintentionally reduce their own utility. A central open question is therefore how DMs should communicate explanations that avoid harming strategic agents while still supporting their own goals, e.g., minimising predictive error. In this work, we analyse widely used explanation methods and establish a necessary condition to prevent explanations from inducing self-harming responses. Furthermore, we show that \emph{action recommendation-based explanations}  (\ARexes), which encompass counterfactual explanations, are sufficient to induce all non-harmful responses. 
Under a conditional homogeneity assumption, this sufficiency extends to \ARex-generating methods, echoing the revelation principle in information design. To demonstrate their practical utility, we introduce a simple learning procedure that jointly optimises the predictive model and the explanation-generating policy. Experiments 
% on both synthetic and real-world tasks 
show that \ARexes enable DMs to achieve high predictive performance while preserving agents' utility, offering a principled strategy for safe and effective partial model disclosure.\textsuperscript{\dag}
\end{abstract}

\begingroup
\renewcommand{\thefootnote}{\dag}
\footnotetext{Published at AISTATS 2026; see the \href{https://openreview.net/forum?id=TpimCQwkR9}{conference version.}}
% \addtocounter{footnote}{-1}
\endgroup

\section{Introduction}
Modern regulatory frameworks emphasise transparency in algorithmic decision making, mandating that decision makers (DMs) provide clear justifications for automated decisions~\citep{selbst2017meaningful,wachter2017right}. For example, the General Data Protection Regulation (GDPR)~\citep{gdpr} includes provisions commonly referred to as the \textit{right to explanation}, requiring DMs to inform agents (i.e., individuals affected by DMs' decisions) about the basis of these decisions in a comprehensible manner~\citep{goodman2017european}. These provisions aim to help agents both understand and contest algorithmic decisions.
However, transparency can incentivise agents to manipulate their inputs to secure more favorable outcomes, triggering strategic adaptations by both agents and DMs~\citep{hardt2016strategic}. This dynamic has motivated extensive research into modeling strategic behavior and optimising decision making under such interactions~\citep{miller2020strategic}. 
Within this line of work, explainability is often equated with full disclosure of the decision making model, including its architecture and parameters~\citep{shavit2020causal,harris2022strategic,vo2024causal}. Such disclosure ensures transparency by enabling agents to simulate and assess alternative actions.
% This perspective assumes that full disclosure of the predictive model inherently satisfies the need for transparency, enabling agents to simulate and assess different scenarios using the disclosed information.

%%%%%
\begin{figure*}[t!]
    \centering
    \includegraphics[width=0.94\textwidth]{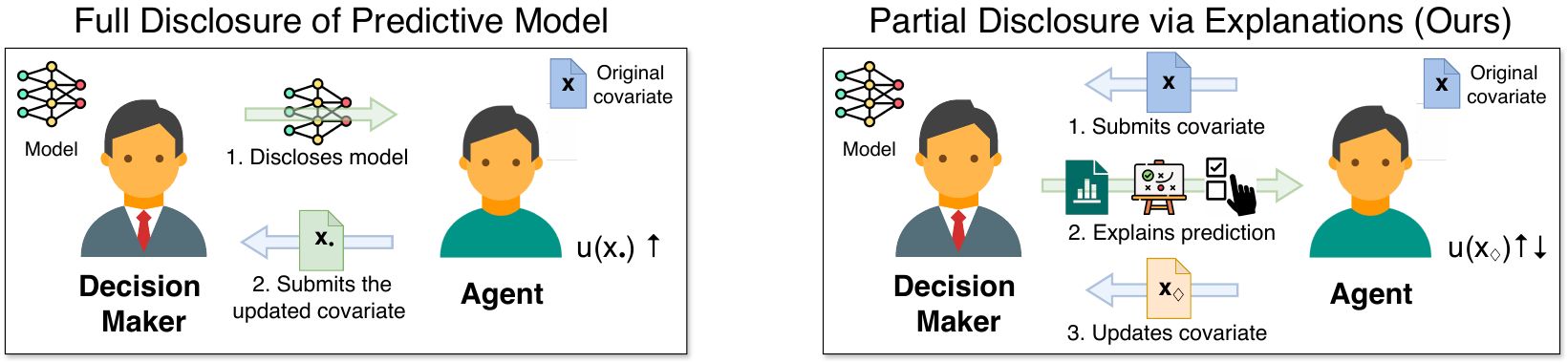}
    % \begin{subfigure}[b]{0.48\textwidth}
    %     \includegraphics[width=0.99\textwidth]{figures/strategic-explanation-Page-10-a.drawio.pdf}
    % \end{subfigure}
    % \hfill
    % \begin{subfigure}[b]{0.51\textwidth}
    %     \includegraphics[width=0.99\textwidth]{figures/strategic-explanation-Page-10-b.drawio.pdf}
    % \end{subfigure}
    \caption{
    \small
        (left) With full access to the DM's model, the agent with a covariate $x$ can \textit{correctly} anticipate how changes affect predictions and choose a response $x_\bullet$ that reliably improves utility $u(x_\bullet)$. (right) With only an explanation, the agent's response $x_\diamond$, based on partial information, might not improve utility $u(x_\diamond)$.
    }
    % \label{fig:enter-label}
    % \vspace{-10pt}
\end{figure*}
%%%%%

Although full disclosure may formally satisfy transparency requirements, it often pushes DMs towards adopting inherently interpretable models at the expense of predictive performance~(e.g. linear models). For complex models with billions of parameters, full disclosure rarely translates into genuine interpretability for agents and may instead overwhelm them with excessive detail.  This raises important questions about whether full disclosure truly serves the original intent of the right to explanation. Moreover, it may conflict with the interests of DMs, especially when sensitive intellectual property is involved. 
For example, in car insurance pricing, mandating full disclosure could expose a company's proprietary model to competitors, who could exploit it without bearing the development costs, eroding the firm's competitive edge.

% Disclosing the entire model, while meeting regulatory requirements, could jeopardise the competitive advantage of the DM. 

In practice, explaining a DM's predictive model does not necessitate full disclosure. 
Instead, conveying partial information can also be considered~\citep{christoph2020interpretable}.
% Instead, a variety of explanation methods in machine learning focus on conveying partial information~\citep{molnar2020interpretable,christoph2020interpretable}.
While many explanation methods have been studied in strategic learning~\citep{tsirtsis2020decisions, xie2024non, cohen2024bayesian}, it remains unclear which approach would best serve DMs.
Specifically, when explanations omit certain details of the underlying model, agents may misinterpret the model's behaviour and take actions that result in suboptimal or even detrimental outcomes for them (e.g., see \Cref{apx:misled-agent-example}). 
Popular explanation methods typically highlight certain model characteristics\footnote{We provide a more detailed discussion contrasting the \textit{epistemic} and \textit{strategic roles} of explanations in \Cref{apx:epistemic-vs-strategic}.} rather than guiding agents' strategic responses~\citep{lundberg2017unified,tsirtsis2020decisions, chau2022rkhs, chau2023explaining}. As a result,
explanations may inadvertently mislead agents into harmful actions, eroding trust and creating perceptions of unfairness or unreliability.
This motivates a new challenge: designing explanations for strategic agents that prevent harm while supporting the DM’s objectives across diverse decision-making settings. By analogy to information design~(c.f. Section~\ref{sec: related work}), we call this problem \emph{Explanation Design}.
% This raises a new challenge in designing explanations for strategic agents: identifying a class of explanations that avoids harming agents while remaining effective for the DM across a range of decision-making objectives. Drawing an analogy to the field of information design~(c.f. Section~\ref{sec: related work}), we refer to this emerging problem as \emph{Explanation Design}.
As a result, our work addresses the following question:

\emph{When full disclosure of the predictive model is neither feasible nor desirable, what kind of explanations should be communicated to ensure transparency while balancing the interests of all parties involved?}

\textbf{Our contributions.}  We develop a theoretical framework for analysing explanation design in strategic learning, with four main contributions:
\begin{enumerate}[label=(\roman*)]
    \item We show that, unlike agents with full access to the predictive model, those who rely solely on explanations may misinterpret the model's behaviour and take actions that reduce their utility.
    
    \item Under common assumptions on agents’ behavior, we derive a necessary condition (\Cref{theorem:surrogate-necessary,corollary:general-necessary}) for explanations to avoid misleading agents into utility-harming actions.
    % into taking utility-harming actions. 
    In contrast,
    % to explanations that require such condition, 
    we show that counterfactual explanations~\citep{wachter2017counterfactual}, by design, do not mislead agents (\Cref{remark:ar-noharm}).
    
    \item 
    % Our theoretical contributions culminate in \Cref{subsec:action-rec}, where 
    We formalise the class of \textit{action recommendation-based explanations} (\ARexes, singular: \ARex), which encompass counterfactual explanations and show that this class is sufficient to induce all non-harmful actions (\Cref{proposition:ar-sufficiency}). Under conditional homogeneity of agents' responses, this sufficiency extends to \ARex-generating methods (\Cref{theorem:ar-method-sufficiency}), thus offering a principled framework for designing safe explanations in strategic learning.
    
    \item We demonstrate the practical value of \ARexes in synthetic and real-world optimisation tasks, where the DM jointly learns the predictive model and \ARex policy. Results show that \ARexes enable the DM to improve predictive performance while preserving agents' utility.
\end{enumerate}

All proofs are provided in \Cref{apx:proofs}.

\section{Problem Formulation}\label{sec:problem-formulation}
\textbf{Notations.} Random variables are denoted with uppercase letters (e.g., $X$) and their realisations with lowercase letters (e.g., $x$). The index set $\{1,\ldots,T\}$ is denoted as $[T]$. 
We use $\mathcal{X}, \mathcal{Y}$, and $\mathcal{Z}$ to denote the spaces of agents' observed covariates, outcomes, and unobservable, respectively. The model class and explanation space are respectively denoted by $\mathcal{G}$ and $\mathcal{E}$. We use double dot to denote the initial value of agent $t$'s variable (e.g., $\basewt{x}$), as opposed to the value after it has been shifted, due to strategic behaviour (e.g., $x_t$). 
% We use $\mathcal{X}$ and $\mathcal{Z}$ to denote the respective spaces of agents' covariates that are observed and unobserved to the DM. $\mathcal{Y}$ denotes the outcome space of agents and $\mathcal{G}$ denotes the hypothesis space of the DM. More details about their interactions are given in the setup below.

We use the car insurance pricing \citep{shavit2020causal} as our running example and consider the scenario in which a DM (an insurer), interacts with a population of agents (customers), indexed by the integer $t$. Agents interact with the DM separately and independently. For simplicity, we describe the setup for a single agent.
Let $(\base{X},Z)$ be random variables jointly distributed as $P_{\base{X},Z}$. For each agent $t$, let $(\basewt{x},z_t)$ denote an independent realisation of $(\base{X},Z)$. Here, $\basewt{x}\in\mathcal{X}$ represents agent's observable features (e.g., driving records, car's model and features), while $z_t\in\mathcal{Z}$ captures unobservable factors (e.g., socioeconomic factors).
Furthermore, the agent has an unrealised outcome $\basewt{y}:=h(\basewt{x}, z_t)$, where $h:\mathcal{X}\times\mathcal{Z}\to\mathcal{Y}$ is a deterministic potential outcome function and $\basewt{y}\in\mathcal{Y}\subseteq\mathbb{R}$. In our car insurance context, $\basewt{y}$ reflects the future accident cost of this customer, which the DM tries to predict to determine the insurance premium. 

In the beginning, the DM selects a predictive model $g$ from the hypothesis class $\mathcal{G}\subseteq\{g^\prime:\mathcal{X}\to\mathcal{Y}\}$, to approximate $h$, e.g., by learning $g$ from historical data. The agent first submits their base covariate $\basewt{x}$, receiving a preliminary prediction $\hat{\base{y}}_t:=g(\basewt{x})$. In addition, the DM provides an explanation $e_t\in\mathcal{E}$ describing how $\hat{\base{y}}_t$ was computed. In practice, this means a customer submits an insurance application, receives a quoted premium, and is shown an explanation of how that premium was determined.
In this work, we formalise explanations and associated concepts as follows.

\begin{definition}\label{def:explanation-method}
    An \textit{explanation method} is a tuple $(\mathcal{E}, \sigma)$ where $\mathcal{E}$ is the space of feasible explanations and $\sigma:\mathcal{X}\times\mathcal{G}\to\mathcal{E}$ is an \textit{explanation policy} that picks an \textit{explanation} $e\in\mathcal{E}$ for the agent with base covariate $x\in\mathcal{X}$ w.r.t. the predictive model $g\in\mathcal{G}$. The explanation $e$ is a global explanation if $\sigma$ is a constant function w.r.t. the input $x$. Otherwise, $\sigma$ is said to generate local explanations.
\end{definition}

This definition\footnote{We note that this definition is intentionally broad and places no restriction on the epistemic value of explanations (i.e., their ability to support model understanding). See \Cref{apx:explanation-methods} for further discussion.} allows us to incorporate a wide range of local and global explanation methods and analyse their impact on guiding agent behavior, as we explore in \Cref{sec:agent-responses}. The explanation space $\mathcal{E}$ is general and depends on the chosen explanation method. For example, when using a global surrogate model, such as a linear function, to approximate the predictive model $g$~\citep{molnar2020interpretable}, 
the explanation space $\mathcal{E}$ is a subset of linear functions $\mathcal{F}=\{f:\mathcal{X}\to\mathcal{Y}\,:\,f(x)=w^\top x+b\}$. Alternatively, attribution-based methods such as SHAP~\citep{lundberg2017unified} can have $\mathcal{E}\subseteq\mathbb{R}^d$ where each explanation $e$ is a $d$-dimensional vector of feature attributions. Several other explanation methods fit within this framework, and we provide a more comprehensive discussion of them in \Cref{apx:explanation-methods}.

\textbf{Agent's reaction.} 
After receiving an explanation $e_t$, the agent seeks to improve their predictive score by strategically modifying their features before resubmitting~\citep{tsirtsis2020decisions,xie2024non,karimi2021algorithmic,harris2022bayesian}. For example, an insurance customer might take actions such as installing a telematics device or upgrading to a safer vehicle so as to lower their predicted risk and premium when reapplying for a contract.
% After receiving an explanation $e_t$, the agent may seek to improve their predictive score by strategically modifying their input features before resubmitting. For instance, an insurance customer might update their profile to obtain a lower premium after reviewing the explanation. This type of behavior has been considered in strategic learning~\citep{tsirtsis2020decisions,xie2024non} and algorithmic recourse~\citep{karimi2021algorithmic, harris2022bayesian}.
These actions often come with tangible costs such as time, money, or effort. The agent must weigh the benefit of a more favourable prediction against the cost of implementing changes. 
Following standard practice in strategic learning~\citep{hardt2016strategic,shavit2020causal,harris2022strategic,vo2024causal}, we model this trade-off with an additive utility function $u_t:\mathcal{G}\times\mathcal{X}\to\mathbb{R}$:
% Following standard practice in strategic learning~\citep{hardt2016strategic,shavit2020causal,harris2022strategic,vo2024causal}, we assume that the agent’s utility $u_t:\mathcal{G}\times\mathcal{X}\to\mathbb{R}$ has the following additive structure:
\begin{align}\label{eq:ag-true-utility}
\scalebox{0.94}{$
u_t(g,x) := b(g, x) - c_t(\basewt{x},x) := (-g(x)) - c_t(\basewt{x},x),
$}
\end{align}
where $b(g,x)$ denotes the potential benefit associated with the prediction $g(x)$, while $c_t$ captures the cost of changing features from $\basewt{x}$ to $x$. In our example, the agent prefers lower prediction scores, hence $b(g,x):=-g(x)$. However, the framework straightforwardly generalises to settings where higher scores are desirable, e.g., in credit scoring.
% where $b(g,x)$ denotes the potential benefit associated with the prediction score $g(x)$, and $c_t$ captures the agent's cost—such as monetary expense or effort—of modifying their features from $\basewt{x}$ to $x$. In our running example, a lower prediction score $g(x)$ corresponds to a lower insurance premium, hence more beneficial for the agent. However, this can be straightforwardly adapted to other settings where agents benefit from higher prediction scores, e.g., credit scoring.

We model heterogeneity in agents’ cost functions by introducing a random function, i.e., function-valued random variable, $C\sim P_C$, drawn from a distribution $P_C$ over a family of cost functions. We allow $C$ to be statistically correlated with $\base{X}$ and $Z$. For example, a customer's ability to modify their features, e.g., upgrading the vehicle, may depend on personal characteristics such as income or socioeconomic background, captured in $\base{X}$ and $Z$. Each agent has a different cost function $c_t$: a realisation of $C$ and unknown to the DM.
% Similarly, we assume the DM does not have access to the agent's cost function $c_t$ and model its uncertainty by letting it be governed by a random variable $C\sim P_C$, for some distribution $P_C$ over a family of cost functions. $C$ might be correlated with $\base{X}$ and $Z$. The cost function $c_t$ is a realisation of $C$.
% \textcolor{red}{allow these cost functions to be heterogeneous by letting the respective random variable $C\sim P_C$ where $C$ might be correlated with $\base{X}$ and $Z$}. 
We further define the cost function $c_t$ as follows:

\begin{definition}[Cost function]\label{def:cost-function}
A function $c_t:\mathcal{X}\times\mathcal{X}\to[0,\infty)$ is a cost function for agent $t$ if it satisfies
$c_t(\basewt{x}, \basewt{x}) = 0$ and 
$c_t(\basewt{x}, x) > 0$ for all $x\neq \basewt{x}$.
\end{definition}

In classical strategic learning \citep{hardt2016strategic}, an agent with full knowledge of $g$ chooses a \textit{best response} $x$ that maximises $u_t(g,x)$. 
Here, without access to $g$, they instead respond to the provided explanation $e_t$.
% However, since agents do not observe $g$, they respond to the provided explanation $e_t$.
% In classical strategic learning setup \citep{hardt2016strategic}, with full knowledge of the DM's predictive model $g$ the agent can best respond by modifying the features from $\basewt{x}$ to $x_t$, which is the maximiser of the utility function in \Cref{eq:agent-utility}. However, without knowing the predictive model $g$, the agent cannot obtain the maximiser $x$ of their utility function $u_t(g,x)$. Instead, their best response has to be based on the provided explanation $e_t$ and their reaction model $\psi$.
Upon receiving $e_t$, the agent modifies their covariate from $\basewt{x}$ to $x_t$. 
% Since agents may respond differently to different types of explanations, 
As agents' behaviour differs across various explanation types,
we define a general reaction model: $x_t:=\psi(\basewt{x},e_t,z_t,c_t)$ where $\psi$ is a deterministic and measurable function. 
After reporting $x_t$, the agent receives the final prediction $\hat{y}_t:=g(x_t)$ and realises the outcome $y_t:=h(x_t,z_t)$, concluding the round. We use $P_X$ to denote the distribution of the random variable $X$ that corresponds to the final covariate $x_t$.

\textbf{DM's objective.} 
A typical objective in strategic learning is to minimise the prediction errors while accounting for the shift in agents' features (e.g., \citep{hardt2016strategic, levanon2021strategic}).
% In standard strategic learning setups, the DM is often assumed to minimise prediction error while accounting for how agents may manipulate their input features (e.g., \citep{hardt2016strategic, levanon2021strategic}). 
Other formulations 
% instead assume that the DM 
aim to maximise agents’ outcomes or welfare \citep{xie2024non,vo2024causal}. In contrast, our work focuses on the problem of designing explanations 
% in the presence of strategic agents, 
and therefore, is agnostic to the DM's objective. Hence, our theoretical results are stated independently of the DM’s goal. \Cref{sec:experiments} illustrates how a DM can apply our framework through example objectives.

\section{Agents' Responses to Explanations}\label{sec:agent-responses}
% The agent does not have direct access to the predictive model $g$, and instead bases their strategic response $x_t$ on the provided explanation $e_t$. 
Without knowing the predictive model $g$, the agent has to base their response $x_t$ on the explanation $e_t$.
Hence, the response may be suboptimal, reducing their true utility (in \Cref{eq:ag-true-utility}). We formalise the desirable agent's behaviour through the following key concept:

% To formalise desirable agent's behaviour, we introduce the following key concept:

% Since the agent does not have access to the actual predictive model $g$, their best response $x_t$, which is influenced by the explanation $e_t$, might be suboptimal and can even lead to a reduction in their utility, in \Cref{eq:ag-true-utility}. Before analysing when and why such phenonmenon occurs, we first introduce a desirable property of an agent's responses: that they should not result in lower utilities, compared to that of the baseline $\base{x}$.

\begin{definition}[Non-harmful responses]\label{def:no-harm}
Let $\nu_t=\left\{x\in\mathcal{X}: u_t(g,x)\geq u_t\big(g,\basewt{x}\big)\right\}$. An agent's response, $x_\bullet$, is a non-harmful response if $x_\bullet \in \nu_t$.
\end{definition}

Our goal is to characterise explanation methods in terms of their impact on agent behavior, identify conditions under which harmful responses may occur, and develop safeguards to prevent such outcomes.
% Our goal is to characterise explanation methods that guide agents toward non-harmful responses. 
We focus on explanation types that are (i) previously studied in strategic learning, or (ii) accompanied by a clear behavioral model specifying how agents react. In particular, we emphasise actionable explanations that enable agents to improve their prediction outcomes, rather than merely offering post-hoc interpretability. 
Many popular explanation methods (e.g., feature attribution techniques such as SHAP) lack such actionable guidance \citep{molnar2020interpretable}, making agent behavior difficult to model. 
% We illustrate this issue with a toy example in \Cref{example:shapley}.
\Cref{example:shapley} illustrates this issue.

% To better understand the extent to which explanation methods might induce harmful responses from agents and to develop preventive measures, we examine several common types of explanation methods in this section. This analysis uncovers patterns and provides a foundation for identifying and characterising classes of explanation methods that ensure only non-harmful responses from strategic agents.

\subsection{Surrogate Models}\label{subsec:surrogate-models}

We begin by analysing the use of surrogate models as explanations, e.g., Taylor expansions~\citep{xie2024non}, in strategic learning. These provide interpretable approximations of the predictive model $g$, enabling the agent to construct a surrogate utility function and plan their responses accordingly.
In what follows, we derive a necessary condition that ensures such explanations do not induce harmful responses.

% Similar to the common best-response model in Strategic Learning literature [REFs], we assume the agent $t$ wants to minimise the predictive score while taking into account the cost and thus has the following utility function:
% \begin{align}
% u_t(g, x) = -\Big(g(x)+c_t(\basewt{x},x)\Big),
% \end{align}
% where $x$ denotes an arbitrary covariate value in $\mathcal{X}$ that this agent $t$ can modify $\basewt{x}$ into.

Let $f_t:\mathcal{X}\to\mathcal{Y}$ be a surrogate model of $g$. When the agent observes only $f_t$, it is natural to assume that they act to maximise the surrogate utility: 
$u_t(f_t, x) = (-f_t(x))-c_t(\basewt{x},x)$ by interpreting $-f_t(x)$ as a proxy for benefit instead of $-g(x)$. This assumption is standard when agents lack access to $g$ (see, e.g., \citep{jagadeesan2021alternative,ghalme2021strategic,bechavod2022information,xie2024non}).
Hence, their best response becomes
$x_t := \arg\max_{x}u_t(f_t,x)$. Since the surrogate utility $u_t(f_t,\cdot)$ differs from the true utility $u_t(g,\cdot)$, $x_t$ may reduce the agent’s true utility. To mitigate this risk, we establish a necessary condition:

% Since the surrogate utility function differs from the true utility function defined in \Cref{eq:ag-true-utility}, the agent's best response $x_t$ might lead to a reduction in their true utility. To mitigate such risks, we establish the following necessary condition to safeguard against such situations.

\begin{theorem}[Necessary condition]\label{theorem:surrogate-necessary}
Given an agent $t$ with the base covariate $\basewt{x}$ who best responds against the surrogate utility function $u_t(f_t,\cdot)$. If it holds, for every possible cost function $c_t$ (\Cref{def:cost-function}), that the resulting best response $x_t$ belongs to the non-harmful set $\nu_t$ (\Cref{def:no-harm}), i.e.,
$u_t(g,x_t)\geq u_t(g,\basewt{x})$,
then the following also holds:
\begin{align}\label{eq:surrogate-necessary-cond}
f_t\big(\basewt{x}\big) - f_t(x) \leq g\big(\basewt{x}\big) - g(x) \quad \forall x \in \mathcal{X}_t^{g{\downarrow}},
\end{align}
with $\mathcal{X}_t^{g{\downarrow}}:=\{x : g(x) < g(\basewt{x})\}$ the set of potential responses with lower 
%prediction 
scores for the agent.
% Let $\mathcal{X}_t^{g^{-}}=\{x : g(x) < g(\basewt{x})\}$ be the set of covariate values that reduce the agent's predictive score and $f_t$ be a local surrogate function designed based on $\basewt{x}$, if the condition 
% \begin{align}\label{eq:surrogate-necessary-cond}
% f_t(\basewt{x}) - f_t(x) \leq g(\basewt{x}) - g(x) \quad \forall x \in \mathcal{X}_t^{g^{-}}
% \end{align}
% is violated then there exists a cost function $c_t$ such that the agent's best response $x_t$ (w.r.t $f_t)$ reduces their true utility, i.e.,
% \begin{align}
% u_t(g, x_t) < u_t(g, \basewt{x}).
% \end{align}
\end{theorem}

This theorem says that the surrogate $f_t$ should not overstate the agent’s potential gain relative to the true model $g$. If \Cref{eq:surrogate-necessary-cond} is violated, there exists a cost function $c_t$ under which the agent is incentivised to choose a response $x_t$ whose cost outweighs the actual gain $g(x_t)$, thereby reducing their true utility. Hence, \Cref{eq:surrogate-necessary-cond} is a necessary safeguard: \emph{if it fails, harmful responses are possible}. 
\Cref{apx:misled-agent-example} gives an example to illustrate how an agent can be misled.
% A toy example in \Cref{apx:misled-agent-example} illustrates how an agent can be misled into taking an overly costly action.

% This theorem implies that if the DM provides a surrogate function $f_t$ as an explanation to an agent $t$, but the condition in \Cref{eq:surrogate-necessary-cond} is violated, then there exists a cost function $c_t$ under which the agent is misled into taking a response $x_t$ whose cost outweighs the gain, ultimately resulting in a reduction in their true utility. Consequently, \Cref{eq:surrogate-necessary-cond} is a necessary condition that safeguard against such situations. A toy example in \Cref{apx:misled-agent-example} is provided to illustrate a scenario where an agent is misled into taking an overly costly action.

% and if it holds for all $x\in\mathcal{X}$ instead of $\mathcal{X}_t^{g^-}$, it becomes a sufficient condition

% The following theorem shows that when the surrogate function $f_t$ has larger rate of change than that of $g$, this might give an agent $t$ a false sense of improvement and thus, the action that is optimal for the surrogate objective $u_t(f_t,\cdot)$ might  harm the true objective $u_t(g,\cdot)$ of this agent. We refer to the \Cref{eq:surrogate-necessary-cond} as the necessary condition for ensuring no reduction of an agent's true utility (\Cref{def:no-harm}). Moreover, if the condition in \Cref{eq:surrogate-necessary-cond} holds for all $x\in\mathcal{X}$, then we obtain a sufficient condition.

In practice, this safeguard guides the DM in using surrogate-based explanations. It rules out designs that exaggerate agents’ perceived gains, a problem common to many popular methods such as LIME \citep{ribeiro2016should}, SHAP \citep{lundberg2017unified}, and Taylor expansions \citep{xie2024non}, which do not satisfy \Cref{eq:surrogate-necessary-cond} by design (see e.g., \Cref{subsec:exp-noharm}).
% In conclusion, when using surrogate models as explanations, the DM must ensure these do not exaggerate agents' perceived gains. Many popular explanation methods — such as LIME \citep{ribeiro2016should}, SHAP \citep{lundberg2017unified}, and Taylor expansions \citep{xie2024non} — do not satisfy this condition by design.

Moreover, many \textit{noisy} agent reaction models 
% in strategic learning 
\citep{rosenfeld2020predictions,jagadeesan2021alternative,bechavod2022information} can be interpreted as agents responding to some surrogate function $f_t$.
Our necessary condition thus extends to those settings as well, offering guidance on designing communication strategies to ensure agents' non-harmful responses.
Similarly, this also extends to broader agent models where agents form beliefs about $g$ based on explanations, a setup that is considered in the recent work by \citet{cohen2024bayesian}. \Cref{apx:extended-necessary-condition} discusses this extension.

\subsection{Action Recommendation-based Explanations}\label{subsec:action-rec}

We consider explanations of the form $e_t=(\rec{x}_t, \hat{\rec{y}}_t)$ where  $\rec{x}_t\in\mathcal{X}$ denotes a recommended covariate update suggested by the DM and $\hat{\rec{y}}_t:=g(\rec{x}_t)$ corresponds to the predicted outcome if the agent follows this recommendation. We refer to an explanation of this type as an \textit{action recommendation-based explanation} (\ARex, plural: \ARexes).
This class of explanations is desirable because, as we show later, it is sufficient for inducing agents' non-harmful responses.

Formally, an \ARex policy is a mapping $\sigma:\mathcal{X}\times\mathcal{G}\to\mathcal{X}\times\mathcal{Y}$. A common design choice for $\sigma$ is to recommend a minimal feature modification that achieves a desired prediction, often studied as counterfactual explanations (CEs) in explainable ML literature~\citep{molnar2020interpretable}. Throughout, we use \ARex as a general term for an explanation of the form $(\rec{x}_t, \hat{\rec{y}}_t)$ regardless of how it is generated. 
% When an explanation is produced by a specific instantiation of $\sigma$ consistent with the counterfactual explanations literature, e.g., \citet{wachter2017counterfactual}, we refer to it as a counterfactual explanation. 
When $\sigma$ instantiates a design from the CE literature, e.g., \citet{wachter2017counterfactual}, we may also call the resulting explanations CEs.
This distinction allows us to analyse \ARexes broadly without committing to any particular design of $\sigma$, such as those that minimise feature modifications \citep{molnar2020interpretable} or those that provide \textit{causally plausible} changes in algorithmic recourse~\citep{karimi2022survey}.

% There are various ways to design the explanation policy $\sigma:\mathcal{X}\times\mathcal{G}\to\mathcal{X}\times\mathcal{Y}$ in practice, one of which is to have $\sigma$ recommend a minimal change in an agent's features that allows this agent to obtain the desired prediction value. The explanations generated by such a $\sigma$ is commonly referred to as counterfactual explanations in the explainable ML literature~\citep{molnar2020interpretable}. From here onwards, we use the term AR-based explanation to refer to any explanation of the general form $e_t=(\rec{x}_t, \hat{\rec{y}}_t)$, and when it is clear from the context that $e_t$ is generated by a specific $\sigma$ used in the counterfactual explanations literature (e.g., \citealt{wachter2017counterfactual}), then we call it a counterfactual explanation. This allows us to provide an analysis for this class of explanations without restricting ourselves to any specific assumption on the behaviour of $\sigma$. It also emphasises the fact that the explanation $e_t=(\rec{x}_t, \hat{\rec{y}}_t)$ provided to the strategic agent $t$ is a \textit{factual} guidance, rather than a \textit{counterfactual} scenario.

Compared to the previous subsection where the agent has to infer feature updates based on a surrogate model $f_t$, an \ARex explicitly recommends an action $\rec{x}_t$ and reveals its predicted outcome $\hat{\rec{y}}_t$. Thus, the agent no longer needs to infer the predictive model $g$ or speculate about alternative feature changes.
% Since the DM explicitly reveals the predicted outcome $\hat{\rec{y}}_t$ associated with the recommendation $\rec{x}_t$, it is natural to model the agent as choosing between keeping the original $\basewt{x}$ or adopting the recommended $\rec{x}_t$, depending on which yields higher expected utility. The agent no longer needs to infer the structure of the predictive model $g$ or speculate about alternative feature modifications.
Precisely, we follow \citet{tsirtsis2020decisions} and assume that the agent chooses between keeping $\basewt{x}$ and adopting $\rec{x}_t$ by comparing their utilities:
\begin{align}\label{eq:ar-agent-model}
x_t := 
% \left\{\begin{aligned}
%     &\rec{x}_t\quad \text{if}\ u_t\big(g,\rec{x}_t\big)\geq u_t\big(g,\basewt{x}\big)
%     \\
%     &\basewt{x}\quad \text{otherwise}
% \end{aligned}\right.
\rec{x}_t \; \text{ if }\; u_t\big(g,\rec{x}_t\big)\geq u_t\big(g,\basewt{x}\big),
\; \text{ else }\; \basewt{x}.
\end{align}
That is, the agent adopts the recommendation $\rec{x}_t$ if it improves their utility relative to staying with $\basewt{x}$. Since the DM discloses both $\hat{\base{y}}_t$ and $\hat{\rec{y}}_t$, the agent can evaluate and compare the two utility values directly.
This behavioural model assumes that if $\rec{x}$ does not improve the agent's utility, they will keep their base covariate $\base{x}$ rather than exploring some other action $x^\prime$.
We view this as a minimal rationality assumption: without reliable knowledge of the predictive model $g$, agents avoid experimenting with costly alternatives whose benefits are uncertain.
This setup lets us \textbf{\textit{study the effect of explanations in isolation}} from other side information agents may possess, e.g., from past experience or repeated interactions. Such a focus is both compatible with our research goal and realistic in domains where exploration is costly or restricted, e.g., insurance or credit pricing. \Cref{apx:more-on-choice-model} elaborates on this and discusses alternative behavioral models.

The `$\geq$' sign in \Cref{eq:ar-agent-model} lets the agent break ties in favor of the recommendation $\rec{x}_t$, because in many applications, a more favorable prediction $\hat{\rec{y}}$ is linked to better long-term outcomes (e.g., improved financial status, or lower accident risk in insurance pricing). It follows directly that, with \ARexes, an agent's best response never harms their true utility (\Cref{remark:ar-noharm}) and that the class of \ARexes suffices to induce all non-harmful responses from agents (\Cref{proposition:ar-sufficiency}).

\begin{remark}\label{remark:ar-noharm}
For an agent $t$, any \ARex policy 
$\sigma$
% $\sigma:(\basewt{x},g)\mapsto(\rec{x}_t,\hat{\rec{y}}_t)$ 
will induce a best response $x_\bullet$ that belongs to the set of this agent's non-harmful actions $\nu_t$. 
% This is because $x_\bullet$ is either $\basewt{x}$ or $\rec{x}_t$.
% The response $x_\bullet$ does not have to coincide with $\rec{x}_t$.
\end{remark}

\begin{proposition}[Sufficiency of \ARexes]\label{proposition:ar-sufficiency}
Consider a population of $T$ heterogeneous agents, indexed by $t\in[T]$ and their base covariates $(\basewt{x})_{t\in[T]}$. Suppose agents respond to \ARexes according to \Cref{eq:ar-agent-model}.
% Without loss of generality, agents are ordered by their index $t$. 
% Let $(\basewt{x})_{t\in[T]}$ be the 
% ordered 
% tuple of their base covariates and assume agents respond to \ARexes according to \Cref{eq:ar-agent-model}. 
Then, for any tuple of (different) explanations $(e^\bullet_t)_{t\in[T]}$ that induce non-harmful responses $(x^\bullet_t)_{t\in[T]}$, there exists a tuple of only \ARexes $(e^\diamond_t)_{t\in[T]}$ whose induced responses $(x^\diamond_t)_{t\in[T]}$ satisfy $(x^\diamond_t)_{t\in[T]}=(x^\bullet_t)_{t\in[T]}$.
\end{proposition}

However, this alone does not guarantee that an \ARex policy in practice can generate such a tuple of \ARexes, as agents’ observable features $\base{x}$ might not be informative enough.
Since an \ARex induces at most two actions, i.e., \Cref{eq:ar-agent-model}, it may fail to capture the diversity of responses that other explanations could yield among agents with the same base covariate (see, e.g., \Cref{apx:diversity-of-responses}).
To that end, we introduce the following condition that specifies how much knowledge the DM must obtain, in $\base{x}$, to align agents’ responses. This condition \emph{does not restrict agents' heterogeneity} and is formalised as the following assumption.
% However, this result does not guarantee that there exists an \ARex policy to generate such a tuple, since agents' hidden features can be heterogeneous. To ensure the practicality of this result, we consider the following assumption that acts as a guidance on how much information the DM must obtain, without sacrificing agents' heterogenity.

% Next, to arrive at the sufficiency property of \ARex methods, we introduce the following assumption.
% In the following, we study the sufficiency property of \ARexes. To facilitate the analysis, the following assumption is introduced.

\begin{assumption}[Conditional homogeneity of agents' responses]\label{assumption:subhomo-response}
Given an arbitrary explanation method characterised by $(\mathcal{E},\sigma)$, for any subset of $T^\prime$ agents who share the same base covariate and receive the same explanation, i.e., $(\basewt{x}, e_t)=(\base{x},e)$ for all $t\in[T^\prime]$, their responses must be identical: $x_t=x_\bullet$ for all $t\in[T^\prime]$ and for some $x_\bullet\in\mathcal{X}$.
\end{assumption}

Although appearing strong at first glance, this assumption is in fact weaker than the standard premise in information design~\citep{bergemann2019information} and strategic learning with Bayesian persuasion~\citep{harris2022bayesian} where a DM is expected to know exactly how agents respond so that their recommendation policy induces obedience, i.e., agents always follow recommendations. Instead, we only require that the DM has ``\textit{just enough}" information---captured by $\basewt{x}$---such that, conditional on $\basewt{x}$, the unobserved factors that directly affects agents' responses are homogeneous. 
% Agents may remain conditionally heterogeneous on other aspects.
Thus, this places no restriction on agent heterogeneity.
% ; rather, it requires the DM to obtain sufficient knowledge to conditionally align agents' responses.
\Cref{apx:more-on-conditional-homogeneity} elaborates on these points and gives an example to illustrate their significance.
% of our assumption as a relaxation of the full-information premise.

% Although \Cref{assumption:subhomo-response} appears strong initially, it is a standard premise in information design~\citep{bergemann2019information} and strategic learning with Bayesian persuasion~\citep{harris2022bayesian}. In these works, a decision maker (or equivalently, information designer) is expected to know how agents will react so that their action recommendation policy can induce obedience. In contrast, our \Cref{assumption:subhomo-response} is a weaker condition. Instead of requiring the DM to have full knowledge of the agents' reaction model, it assumes that the DM possesses ``\textit{just enough}" information---captured by $\basewt{x}$---such that when conditioning on such information, the unobserved parts of agents that influencing their responses are homogeneous. This means that while agents may respond in diverse ways when considered across the entire population, their responses become \emph{conditionally} homogeneous given the available information. 

% Consequently, even though their responses are heterogeneous population-wise, \textit{conditionally}, those responses are homogeneous.

In practice, this assumption can be enforced by having the DM query additional information from agents prior to generating explanations. 
For instance, a car insurer might survey customers on whether they prefer completing a defensive driving course or installing a telematics device.
% For instance, a car insurer might present a short survey asking whether a customer would prefer completing a defensive driving course or installing a telematics device. 
Using this information, the insurer can recommend an action aligned with the customers' preferences, helping them obtain a lower premium.

% Thanks to the assumed agents' behaviour as specified in \Cref{eq:ar-agent-model}, AR-based explanations constitute a sufficient class of explanation methods that induce non-harmful agents' responses. Specifically, any non-harmful response (\Cref{def:no-harm}) that is induced by an arbitrary explanation method $\sigma:\mathcal{X}\times\mathcal{G}\to\mathcal{E}$, can also be induced by an AR-based explanation policy $\sigma^\prime:\mathcal{X}\times\mathcal{G}\to\mathcal{X}\times\mathcal{Y}$. We formally state this result in the next theorem.

\begin{theorem}[Sufficiency of \ARex methods]\label{theorem:ar-method-sufficiency}
Suppose conditions in \Cref{proposition:ar-sufficiency} hold. We assume further that \Cref{assumption:subhomo-response} holds for all explanation methods. Then, for any arbitrary explanation method $(\mathcal{E},\sigma)$ that induce non-harmful responses $(x^\bullet_t)_{t\in[T]}$, there exists an \ARex method $(\mathcal{E}^\prime, \sigma^\prime)$ whose induced responses $(x^\diamond_t)_{t\in[T]}$ satisfy $(x^\diamond_t)_{t\in[T]}=(x^\bullet_t)_{t\in[T]}$.
\end{theorem}

% \begin{theorem}[Sufficiency of \ARexes]\label{theorem:ar-sufficiency}
% For a subset of $T^\prime$ agents with the same base covariate $\basewt{x}=\base{x}, \forall t\in[T^\prime]$, let $\nu_t$ be the set of non-harmful responses of agent $t$ (\Cref{def:no-harm}).
% If the agents' responses are conditionally homogeneous (\Cref{assumption:subhomo-response}), then all $T^\prime$ agents must have the same set of non-harmful responses, i.e., $\nu_t=\nu, \forall t\in[T^\prime]$. Moreover, for any explanation $e$ generated by an arbitrary explanation method $(\mathcal{E}, \sigma)$, i.e., $\sigma:(\base{x},g)\mapsto e$, such that $e$ induces a response $x_\bullet\in\nu$, there exists an \ARex method $(\mathcal{E}^\prime,\sigma^\prime)$ such that $\sigma^\prime:(\base{x},g)\mapsto(\rec{x},\hat{\rec{y}})$ induces the same response $x_\bullet$.
% \end{theorem}

This theorem implies that when evaluating the impact of a DM’s predictive model, it suffices to focus solely on \ARexes. To identify an optimal explanation method under the no-harm constraint (\Cref{def:no-harm}), it suffices to search within the space of \ARex methods. Consequently, methods outside this class, e.g., LIME, cannot outperform optimal \ARexes under the no-harm requirement, regardless of the DM’s objective.
This sufficiency property is analogous to that of Bayes correlated equilibria (BCE)\footnote{This result of \citet{bergemann2019information} generalises the idea of \textit{straightforward signal} from Bayesian persuasion~\citep{kamenica2011bayesian} to the multi-agent environment, such signal is also called Bayesian incentive-compatible~\citep{harris2022bayesian}.} in information design \citep{bergemann2019information}, but with an important distinction: BCE imposes a stronger condition. A BCE can be interpreted as an \ARex policy that additionally satisfies an obedience-inducing constraint. Thus, while BCE suffice to rationalise any agent behavior, \ARex methods in our setup suffice to rationalise any \textit{non-harmful} agent behavior.

% This sufficiency property is analogous to the sufficiency of Bayes correlated equilibria (BCE)\footnote{This result of \citet{bergemann2019information} generalises the idea of \textit{straightforward signal} from Bayesian persuasion~\citep{kamenica2011bayesian} to the multi-agent environment, such signal is also called Bayesian incentive-compatible~\citep{harris2022bayesian}.} in information design \citep{bergemann2019information}.
% However, a key distinction is that BCE imposes a stronger condition than our \ARexes: a BCE can be interpreted as an \ARex policy that additionally satisfies an obedience-inducing constraint. Moreover, while the class of BCE suffices to rationalise any agent behavior, the class of \ARexes in our strategic learning setup suffices to rationalise any \textit{non-harmful} agent behavior.

To summarise, \ARexes are theoretically desirable because by design, they prevent agents from being misled into harmful actions. Furthermore, under conditional homogeneity, they form a sufficient class of non-harmful explanation methods. 

\textbf{\ARexes in practice.} 
While our theory holds regardless of the DM's objective, the DM can also design their \ARex policy $\sigma$ to better serve specific goals. For example, they may jointly optimise both $\sigma$ and the predictive model $g$, rather than relying on fixed designs as in standard CE methods \citep{molnar2020interpretable}. This lets the DM tailor $\sigma$ to the task at hand. We illustrate this through two concrete scenarios next.

% As \ARexes form a class of good enough explanations, regardless of the DM's objective. However, one can go one step further and optimally design the explanation policy $\sigma$ within the class of \ARexes, to suit their specific decision-making goals. We illustrate such practical relevance of the broad class of \ARexes through two concrete scenarios in \Cref{subsec:exp-synthetic} and \Cref{subsec:exp-german-credit}.

\section{Empirical Studies}\label{sec:experiments}

We empirically evaluate two key properties of \ARexes: (i) their no-harm guarantee (\Cref{remark:ar-noharm}), and (ii) their practical value when optimised for specific objectives. \Cref{subsec:exp-noharm} validates the no-harm guarantee by comparing \ARexes against two surrogate-based methods.
% Taylor expansions, a representative surrogate method. 
\Cref{subsec:exp-synthetic,subsec:exp-german-credit} then show how the DM can improve predictive performance by jointly optimising $g$ and $\sigma$.
\Cref{apx:detailed-exp} provides the full setup details.
% We briefly describe the setups here and leave full details to \Cref{apx:detailed-exp}.

% As noted in \Cref{subsec:action-rec}, standard methods for generating counterfactual explanations, such as those reviewed by \citet{molnar2020interpretable}, correspond to specific design choices for the \ARex policy $\sigma$. Here, we explore an alternative approach: jointly optimising the predictive model $g$ and explanation policy $\sigma$ to minimise the expected predictive error\footnote{Our learning procedure is intended to extend beyond the classification and discrete setting considered in \citet{tsirtsis2020decisions}, aiming for broader applicability.}. This illustrates how the DM can operationalise the broader class of \ARexes in practice.
% Full implementation details and a concrete algorithm are deferred to \Cref{apx:detailed-exp}.
% We provide brief descriptions of our experiments here and leave full details (including a concrete algorithm) to \Cref{apx:detailed-exp}.

\subsection{On the No-harm Guarantee of \ARexes} \label{subsec:exp-noharm}
With a synthetic experiment, we demonstrate that \ARexes guarantee no harmful agents' responses whereas Taylor expansions (used as surrogates~\citep{xie2024non}) and LIME~\citep{ribeiro2016should} do not.
While LIME is practical and comes with an official Python package, Taylor expansions as explanations (Taylor-ex) provide a clear reaction model, allowing exact computation of agents' responses.
% We include Taylor expansions as a baseline because (i) it has a clear reaction model, unlike other explanation methods, and (ii) agents' responses can be computed exactly, allowing for precise experimental result.
% We begin with a synthetic experiment to validate the no-harm guarantee: Taylor expansions used as explanations can lead to reduced agents' utility, whereas \ARexes preserve utility, even when generated arbitrarily.
We use a quartic function as the predictive model of the DM where $g(x)=x^4-x^2+1$, with 2nd-order Taylor expansions as surrogates. 
For LIME, we construct a linear approximation of $g$ for each agent.
For simplicity, \ARexes are generated randomly.
We generate 100 agents with a scalar feature $\basewt{x}\in\mathbb{R}$ and use the cost function $c_t(\basewt{x},x)=|\alpha_t|\|\basewt{x}-x\|_2^2$, with $\alpha_t\in\mathbb{R}$ being the cost factor, reflecting heterogeneity in agents. 
% We compare the result of Taylor-ex and LIME against \ARexes, which we generate randomly, for simplicity.

\textbf{Results.}
\Cref{subfig:taylor-harm} plots the change in agents' utility values before and after performing best responses. Taylor-ex and LIME respectively misled $49\%$ and $95\%$ of agents to reducing their utility. Although these two approximate local structures of $g$, they may exaggerate agents' gains and mislead them into taking costly actions. In contrast, \ARexes do not mislead agents, even when the recommendations are generated arbitrarily.

\begin{figure*}[t]
    \centering
    \begin{subfigure}[t]{0.34\textwidth}
        \includegraphics[width=0.98\textwidth]
        {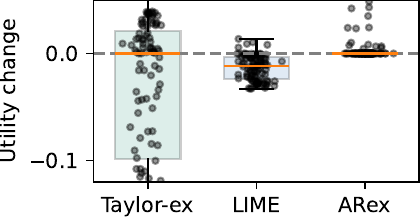}
        \subcaption{Agents' utility change.}
        \label{subfig:taylor-harm}
    \end{subfigure}
    \hfill
    \begin{subfigure}[t]{0.29\textwidth}
        \includegraphics[width=0.98\textwidth]
        {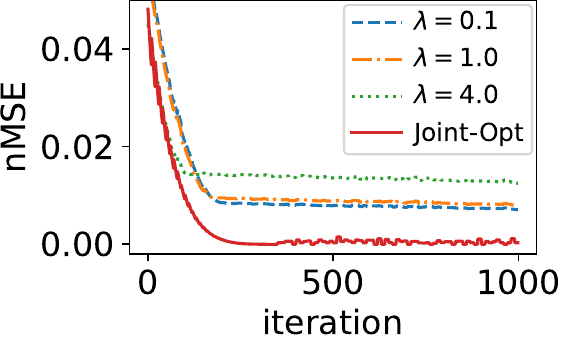}
        \subcaption{Synthetic data.}
        \label{subfig:synthetic-plot}
    \end{subfigure}
    \hfill
    \begin{subfigure}[t]{0.28\textwidth}
        \includegraphics[width=0.98\textwidth]
        {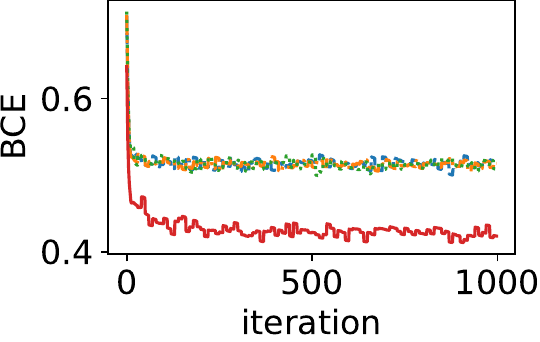}
        \subcaption{German credit data.}
        \label{subfig:german-credit-plot}
    \end{subfigure}
    \caption{
        \small 
        (\subref{subfig:taylor-harm}) Taylor-ex and LIME mislead agents into reducing their utility while \ARexes do not. The box plot shows the change in agents' utility after best responding.
        (\subref{subfig:synthetic-plot}) $\&$ (\subref{subfig:german-credit-plot}): \ourprocedure has the lowest training-loss curves (nMSE and BCE) against the three baselines, showing that jointly optimising both $g$ and $\sigma$ is more beneficial to the DM. 
    }
    \label{fig:result}
    % \vspace{-1em}
\end{figure*}

\subsection{A Synthetic Experiment}\label{subsec:exp-synthetic}

Next, we study a synthetic insurance pricing task that mirrors the setup in \Cref{sec:problem-formulation}. The DM aims to choose the predictive model $g$ and the \ARex policy $\sigma$ that jointly minimise the expected prediction error. We refer to our approach as \textit{joint-optimisation} (\ourprocedure).\footnote{We provide an example algorithm in \Cref{apx:detailed-exp}. Our learning procedure extends beyond the classification and discrete setting in \citet{tsirtsis2020decisions}, aiming for broader applicability.}

% In both settings, the DM deploys predictive model $g$ and provides agents with explanations generated by a policy $\sigma$. Agents then respond strategically depending on the given explanations, seeking to improve their prediction scores while incurring minimal cost. Given such strategic behaviour, DM seeks to choose the optimal pair of $g,\sigma$ such that together, they minimise the predictive error. Additional details are given in respective experiments and full details are provided in \Cref{apx:detailed-exp}.

We construct a dataset with agents of a 3-dimensional (observable) feature vector $\basewt{x}\in\mathbb{R}^3$ and a scalar (unobservable) feature $z_t\in\mathbb{R}$. Each agent $t$ has the cost function $c_t(\basewt{x},x)=|\alpha_t|\|\basewt{x}-x\|_2^2$ where $\alpha_t$ is correlated with $\basewt{x}$ and $z_t$. The outcome function $h:\mathbb{R}^4\to\mathbb{R}$ is a quadratic function of the concatenated features $(x_t,z_t)$.
Here, the DM's goal is to learn $g$ and $\sigma$ that jointly minimise the expected squared loss:
$\min_{g,\sigma}\mathbb{E}_{P_{X,Z}}[(g(X)-h(X,Z))^2]$
where the DM's choice of $\{g,\sigma\}$ affects $P_X$, the distribution of agents' responses\footnote{For brevity, we omit the explicit connection between $X$ and its optimisation arguments from the objective function, though a full expansion is provided in \Cref{subapx:expanded-obj}.}. This objective reflects a typical goal in insurance pricing: setting premiums that align with future accident costs. Underpricing may lead to financial losses, while overpricing could drive customers away. However, jointly optimising $\{g,\sigma\}$  is challenging due to their interdependence (see \Cref{def:explanation-method}). To address this, observe that for any given pair $\{g,\sigma\}$ that generates \ARexes, there exists an equivalent pair $\{g,\sigma^r\}$ that could generate the same, where we call $\sigma^r:\basewt{x}\mapsto \rec{x}_t$ an \textit{action recommendation function}. This lets us rewrite the loss as a function of $(g,\sigma^r)$, yielding:
% we decompose $\sigma:(\basewt{x},g)\mapsto(\rec{x}_t,g(\rec{x}_t))$ into two disconnected components: the predictive model $g$ and the \textit{action recommendation} (AR) function $\sigma^r:\basewt{x}\mapsto \rec{x}_t$.  This separation allows us to rewrite the loss as a function of $(g,\sigma^r)$, yielding the equivalent objective:
\begin{align} \label{eq:dm-objective-with-ar}
\textstyle
\min_{g,\sigma^r}\ \mathbb{E}_{P_{X,Z}}[(g(X)-h(X,Z))^2] .
\end{align}
In practice, because the DM have no access to the outcome function $h$ and the unobserved $z_t$, an efficient way to solve \Cref{eq:dm-objective-with-ar} is through the repeated risk minimisation (RRM)~\citep{perdomo2020performative}. While training $g$ in RRM is straightforward, the same does not apply to $\sigma^r$ because each step of risk-minimisation only works if the prediction, i.e., $\hat{y}_t:=g(x_t)$, changes whenever $\sigma^r$ is updated. Thus, the DM must simulate how agents adapt their responses $x_t$ to changes in $\sigma^r$. 
As part of the learning procedure, the DM can deploy some (possibly arbitrary) models $\{g,\sigma\}$ to obtain agents' responses, then learn a model $\hat{\psi}: (\basewt{x}, \hat{\base{y}}_t, \rec{x}_t, \hat{\rec{y}}_t)\mapsto \hat{x}_t$ that predicts agents' responses, similar to the work of \citet{xie2024non}. We defer complete details and algorithm to \Cref{apx:operation-on-synthetic-experiment}.

\textbf{Baselines.}
We consider counterfactual (CE) policies of the form: $\rec{x}_t := \arg\min_{x}(g(x) + \lambda\|x-\basewt{x}\|_2^2)$ \& $\hat{\rec{y}}_t := g(\rec{x}_t)$ and instantiate several values for $\lambda$ to act as different baselines.
% We consider the following form of counterfactual explanation (CE) policies and instantiate several values for $\lambda$ to act as different CE baselines:
% \begin{align*}
% \textstyle
% \rec{x}_t := \arg\min_{x}(g(x) + \lambda\|x-\basewt{x}\|_2^2),
% \quad
% \hat{\rec{y}}_t := g(\rec{x}_t).
% \end{align*}
Here, $g$ is trained with RRM, but the CE policies remain fixed.
% Similar to our main approach, $g$ is trained with RRM, but here the CE policies remain fixed.
No other CE variants are considered, since \ARexes subsume CEs and the aim of this experiment is to evaluate joint optimisation of the \ARex policy and $g$, rather than to search for the best CE variant.
Since the CE policies are fixed, agents' responses are not simulated here.

\textbf{Evaluation.} We compare the prediction errors on a test set of $10^6$ strategic agents. For ease of presentation, the mean-squared errors (MSE) are scaled by a constant and reported as normalised MSE (nMSE): computed on offline data they simply are nMSE, and with agents' strategic responses, \emph{strategic nMSE}.

% \textbf{Evaluation.} We then compare the prediction errors between all approaches on a hold-out test set of $10^6$ strategic agents. For ease of presentation, we scale the mean-squared errors by dividing them with the constant $nc=\frac{1}{T}\sum_{t\in[T]}\basewt{y}$ that is independent of the DM's choice of models. We call the scaled errors normalised mean-squared errors (nMSE): computed on offline data they are nMSE, and with agents' strategic responses, \emph{strategic} nMSE.

\begin{table*}
    \centering
    \caption{\small \ourprocedure achieves the lowest test error on synthetic data and highest test score on real-world data, while maintaining strong compliance rates, which indicates the portions of agents that follow the DM's recommended actions.}
    \label{tab:result-on-test}
    \scalebox{0.9}{
    \begin{tabular}{llcccc}
    \toprule
    & & {\ourprocedure} 
      & $\lambda=0.1$ 
      & $\lambda=1.0$ 
      & $\lambda=4.0$ \\
    \midrule
    \multirow{2}{*}{Synthetic data} 
      & Strategic nMSE ($\downarrow$)     & \textbf{2e-4} & 7e-3 & 8e-3 & 1e-2 \\
    & Compliance rate     & 1.0 & 1.0  & 1.0  & 1.0   \\
    \midrule
    \multirow{2}{*}{German credit data} 
      & Strategic $F_1$ score ($\uparrow$)     & \textbf{0.9} & 0.86 & 0.84 & 0.85   \\
    & Compliance rate     & 0.84 & 1.0 & 1.0 & 1.0   \\
    \bottomrule
    \end{tabular}
    }
    % \vspace{-1em}
\end{table*}

% \begin{table}
%     \centering
%     \caption{Our approach, $(g_\text{loc},\sigma_\text{loc}^r)$, achieves the lowest test error and highest test score, without sacrificing persuasive power. The compliance rates show the portions of agents that follow the DM's recommended actions.}
%     \label{tab:result-on-test}
%     \renewcommand{\arraystretch}{1}
%     \begin{tabular}{llccc}
%     \toprule
%     & & Ours 
%       & Global AR 
%       & CE \\
%     \midrule
%     \multirow{2}{*}{Synthetic data} 
%       & Strategic nMSE $(\downarrow)$    & \textbf{2e-4} & 123e-4 & 241e-4 \\
%     & Compliance rate      & 1.0           & 0.3    & 1.0    \\
%     \midrule
%     \multirow{2}{*}{German credit data} 
%       & Strategic $F_1$ score $(\uparrow)$    & \textbf{0.93} & 0.83   & 0.81   \\
%     & Compliance rate            & 0.94          & 0.18   & 1.0    \\
%     \bottomrule
%     \end{tabular}
%     \renewcommand{\arraystretch}{1}
% \end{table}

% \begin{table}[t!]
%     \centering
%     \caption{Our approach achieves the lowest test error, without sacrificing persuasive power. The compliance rates show the portions of agents that follow the DM's recommended actions.}
%     \label{tab:result-on-test}
%     \begin{tabular}{c|c|c|c|}
%     \cline{2-4}
%     & $g_\text{loc}$ \& $\sigma^r_\text{loc}$ & $g_\text{glo}$ \& $x^r$ & $g_\text{ce}$ \& $\sigma^r_\text{ce}$ \\ \hline
%     \multicolumn{1}{|c|}{strategic nMSE} & \textbf{2e-4} & 123e-4 & 241e-4 \\ \hline
%     \multicolumn{1}{|c|}{compliance rate} & \textbf{1.0} & 0.3 & 1.0 \\ \hline
%     \end{tabular}
%     % \vspace{-1em}
% \end{table}

\textbf{Results.}
\Cref{subfig:synthetic-plot} shows the training loss (nMSE) under RRM and \Cref{tab:result-on-test} reports test performance. Although all methods optimise the predictive model $g$ while accounting for agents' strategic behaviour via RRM, our results show that the choice of explanation policy $\sigma$ plays a critical role. Specifically, varying $\lambda$ in fixed CE policies already impacts the predictive performance of $g$, and jointly optimising $\sigma$ with $g$ yields further improvement. This highlights the benefit of learning a non-harmful explanation policy tailored to the DM's objective instead of relying on fixed designs.

Although not the main objective, \ourprocedure also achieves full compliance in this synthetic setup. This occurs because the \ARex policy $\sigma$, through optimisation, learns to \textit{convincingly} guide agents towards regions where $g$ performs well.
% , thus becoming effective at inducing obedience in agents. 
While this notion of compliance is simulated rather than real human behaviour, it suggests that optimising explanations with strategic dynamics in mind can be effective. Evaluating this effect in user studies or behavioural experiments is a promising direction for future work.

% In addition, not only our approach has the lowest prediction error, when facing strategic agents, \textcolor{red}{it also achieves perfect compliance from agents. This high compliance rate suggests that although \textit{persuasion} is not the focus of our approach, we can still achieve it as the optimised AR policy favours recommendations that could help agents achieve self-improvement.}

\subsection{German Credit Dataset}\label{subsec:exp-german-credit}

% Here, we follow \citet{xie2024non} to preprocess the  German credit data \citep{statlog_(german_credit_data)_144} and to simulate strategic behaviour.
For the German credit dataset \citep{statlog_(german_credit_data)_144} with a classification task, we follow \citet{xie2024non} to preprocess the data and to simulate strategic behaviour. 
% First, we remove sensitive features and designate 8 out of 18 features as modifiable for agents.
Next, a logistic regression model is fitted on the original data to estimate $y_t$, and used to simulate outcomes when agents modify their covariates $x_t$.
% We then fit a logistic regression model on the original dataset to estimate $y_t$, and use this model to simulate respective outcomes when agents modify their covariates $x_t$. 
We use CTGAN \citep{ctgan} to generate 9,000 more samples for training and 1,000 for testing. 

\textbf{DM-agents interactions.} The DM's predictive model is a binary classifier $g(x):=\mathbbm{1}[g_s(x)\geq0.5]$ where the underlying scoring $g_s(x)\in[0,1]$ outputs the predicted probability of a positive outcome. Agents aim to maximise $g_s(x)$, yielding utility $u_t(g_s,x)$.
% We define the true outcome function $h(x_t)$ via a logistic regression model $h_s(x_t)$. We fit the model $h_s(x_t)$ on the complete base dataset containing 1{,}000 observations of $\{x_t, y_t\}$. Let $p_t := h_s(x_t)$ denote the probability score output by this model. The binary outcome is then defined as $Y_t\mid x_t:=h(x_t)\sim\text{Bernoulli}(p_t)$.
% We then use a conditional generative adversarial network (CTGAN) \citep{ctgan} to generate 10{,}000 additional data points (with a 9:1 train-test split) for our experiment where all the outcomes are generated with $y_t:=h(x_t)$. We construct a scoring function $g_s$ using a 3-layer ReLU net and let the DM use the binary classifier $g$ as their predictive model, where $\hat{y}_t:=g(x_t):=\mathbbm{1}\left[g_s(x_t)\geq0.5\right]$.
% Following \citet{xie2024non}, we let agents to respond against the scoring function $g_s$ and their utility function for a strategic response $x_t$ becomes $u_t(g_s, x_t)$. An agent's best response then follows the choice model in \Cref{eq:ar-agent-model}. 
We design the cost function as
$c(\basewt{x},x) := 0.01\sum_{i\in\mathcal{I}}\left|\base{x}_{ti}-x_{i}\right|/(x_{i}^U-x_{i}^L)$,
where $\mathcal{I}$ is the set of modifiable features and $[x_{i}^L, x_{i}^U]$ denotes the valid range of feature $i$. 
% Any change in a non-modifiable feature incurs infinite cost.

\textbf{Evaluation.} 
% To train $\sigma_{loc}^r$ efficiently, we restrict modifications to actionable features within predefined bounds $[x_i^L, x_i^U]$. 
Counterfactual explanations are generated similarly to the synthetic setup, with additional constraints (e.g., categorical/bounded features) enforced via projected gradient descent. Unlike previous setup, we now use $-g_s(x)$ in the objective function for generating CEs as the agents now benefit from higher prediction scores.
% The training of $g$ follows the same procedure as before, except for the loss functions: we use \textit{weighted} binary cross-entropy (BCE) to learn agents' response function $\xi$ (with imbalanced labels), and \textit{unweighted} BCE in RRM, where outcomes $y_t$ may shift due to strategic behavior.
% The training procedure is the same as the previous synthetic setup, except for the loss functions. For the loss functions, we use a \emph{weighted} binary cross-entropy (BCE) in steps 1 and 2, where labels are unbalanced but fixed, and an \emph{unweighted} BCE in RRM, where outcomes $y_t$ may shift due to agents' strategic behavior.
We evaluate predictive accuracy, with $F_1$ score, on the test set of 1,000 strategic agents, referring to this as \textit{strategic $F_1$ score} to reflect shifts in $(x_t,y_t)$ due to strategic behaviour.
% We then compute the predictive accuracy (using the $F_1$ score) on a hold-out test set of strategic agents. As the true labels $y_t$ can shift due to strategic behaviour, we refer to this score as \textit{strategic $F_1$ score}.

\textbf{Results.}
\Cref{subfig:german-credit-plot} shows the training loss, i.e., binary cross entropy (BCE), under RRM, and \Cref{tab:result-on-test} reports test performance. These results support our findings in the previous synthetic setup: joint optimisation of $g$ and $\sigma$ yields the most favourable outcome for the DM, while maintaining a reasonably high compliance rate under simulated strategic behaviour.

Finally, while our theory focuses on tabular data, \Cref{apx:bone-age-experiment} demonstrates its application to images in a more complex setting.

\section{Related Work}
\label{sec: related work}
In this section, we overview related work and provide complete details in \Cref{apx:extended-related-work}.

% In this section, we highlight closely related work and provide the extended version of this section in \Cref{apx:extended-related-work}.

\textbf{Strategic learning.} 
Several studies have examined how partial information of the DM's predictive model shapes agents’ decisions
% Several studies have examined the effects of sharing partial information about predictive models with agents and analysed how agents make decisions based on this information, 
e.g., \citet{jagadeesan2021alternative, ghalme2021strategic, bechavod2022information, haghtalab2024calibrated, xie2024non}. 
% Several studies in this area have examined the impact of partial information release when agents do not have full knowledge of the predictive model used in decision-making~\citep{jagadeesan2021alternative, ghalme2021strategic, bechavod2022information, harris2022bayesian, haghtalab2024calibrated, xie2024non, cohen2024bayesian}. 
In particular, the works of \citet{harris2022bayesian} and \citet{cohen2024bayesian} are closest to ours. As discussed in \Cref{subsec:action-rec}, \citet{harris2022bayesian} focuses on the obedience-inducing property (Bayesian incentive compatibility) of a subclass of action recommendations, whereas we focus on the no-harm property of action recommendations. 
% As we also discuss in \Cref{subsec:action-rec}, identifying an AR-based explanation policy that can induce obedience for each \textit{individual} agent is hard, especially when agents are heterogeneous (e.g., in \citealt{harris2022strategic,shao2024strategic}), and such \textit{individual}-level identification might not be necessary if the DM only cares about optimising their expected utility, which is computed over the population of agents.
Unlike action recommendations, \citet{cohen2024bayesian} instead releases a subset of the hypothesis class to all agents, aligning with global explanations in our framework (\Cref{sec:problem-formulation}). However, interpreting a set of models—such as neural networks—can be difficult for agents. In contrast, \ARexes are both more interpretable and guaranteed not to mislead agents.

\textbf{Counterfactual explanations (CEs) and algorithmic recourse.} \citet{tsirtsis2020decisions,karimi2022survey} explore CEs and algorithmic recourse, for strategic agents. Although algorithmic recourse focuses on recommending actions to achieve better outcomes, its implementation often relies on strong causal assumptions, which can be impractical when such knowledge is unavailable.
In contrast, our work adopts a weaker notion of desirability centred on agents' welfare---ensuring non-harmful responses---and examines a broader range of explanation types beyond counterfactuals.
Although both \citet{tsirtsis2020decisions} and our work involve CEs, the contributions differ. Precisely, they focus on optimising CEs in strategic settings, while we analyse multiple explanation types and formally show why \ARexes are more desirable. 
In addition, our proposed learning procedure in \Cref{subsec:exp-synthetic}, though not the main focus, is designed to be more general, extending beyond the classification and discrete case in \citet{tsirtsis2020decisions}.

\textbf{Information design.} The extensive literature on information design, surveyed by \citet{bergemann2019information}, studies how to design information disclosure policies in a game of two parties. While our results are inspired by these works, e.g., \Cref{theorem:ar-method-sufficiency}, the goals differ significantly. As discussed in \Cref{subsec:action-rec}, information design aims at \textit{persuading} agents with a general response model and does not necessarily ensure the no-harm property (\Cref{def:no-harm}). In contrast, we study explanation methods that prioritise the no-harm property, safeguarding agents' welfare.
By incorporating specific agent models in strategic settings, we establish the sufficiency of \ARexes without requiring the DM to know exactly how agents respond.
% This allows us to prove the sufficiency property of AR-based explanations without forcing the DM to know about agents' heterogeneous reaction models. 
% Consequently, while the class of BCE suffices to rationalise any agents' behaviour, our setup of AR-based explanations suffice to rationalise any \textit{non-harmful} agents' behaviour.

%%%%
\section{Conclusion}\label{sec:conclusion}

To summarise, we address the challenge of providing safe, actionable explanations in strategic learning, where DMs must balance transparency with utility optimisation. We formalise action recommendation-based explanations (\ARexes), which ensure that agents act without incurring self-harm. By introducing the no-harm property and a conditional homogeneity assumption, we demonstrate that \ARexes let DMs achieve optimal outcomes while safeguarding agent welfare. Consequently, our work clarifies the distinctions of different explanation types through the lens of strategic learning. Finally, we propose a framework to jointly optimise predictive models and \ARex policies, aligning the DM’s objectives with agents' responses. 

Our findings rest on commonly adopted assumptions about agents' behavior, such as their utility functions or reaction models. While these assumptions may limit the generalisability of our approach, they do not diminish its broader relevance. Intuitively, when explanations omit certain information, conditions are needed to ensure agents’ inferred gains are not exaggerated, 
protecting their welfare.
% maintaining realistic and actionable guidance. 
\ARexes succeed in this regard by being transparent about the potential benefits of recommended actions, eliminating the risk of misleading agents.
% focusing exclusively on actionable recommendations that agents can safely follow, thereby ensuring both predictive accuracy for the DM and safety for the agents. 
% This highlights the critical role of coupling utility optimisation with safety guarantees in strategic learning. 
Future work could explore diverse agent behavior models, dynamic environments, and more scalable learning algorithms to enhance the applicability and efficiency of this approach.

\subsubsection*{Acknowledgements}
We sincerely thank the members of the Rational Intelligence (RI) Lab, including Abbavaram Gowtham Reddy, Anurag Singh, Monseej Purkayastha, and Swathi Suhas, for their insightful discussions, constructive feedback, and invaluable contributions to this work. We also extend our gratitude to the visiting researchers, Amin Charusaie, Majid Mohammadi, Masaki Adachi, Rattaya Kaewvichai, and Saptarshi Saha, for their stimulating discussions and fresh perspectives, which enriched our understanding of the problem. Their contributions have been greatly appreciated. 

We also thank the anonymous reviewers for their valuable feedback to improve our work.

Yixin Wang was supported in part by funding from the Office of Naval Research under grant N00014-23-1-2590, the National Science Foundation under grant No. 2310831, No. 2428059, No. 2435696, No. 2440954, a Michigan Institute for Data Science Propelling Original Data Science (PODS) grant, Two Sigma Investments LP, and  LG Management Development Institute AI Research. 

Kiet Q. H. Vo is a doctoral candidate at Saarland University.

\newpage
\bibliographystyle{alp}
\bibliography{components/bib}

\appendix
\onecolumn
\section{Additional Illustrative Examples}
This section contains more examples to illustrate our theory.

\subsection{Examples of local and global explanations}\label{apx:explanation-methods}

In this work, we adopt a broad notion of explanations consistent with prior literature, where an explanation is defined as an element of an explanation space produced by a mapping $\sigma$, and we place no restriction on its epistemic usefulness (i.e., its ability to support model understanding). In particular, some explanation formats—such as counterfactual explanations or partial model descriptions—may offer limited insight into the internal workings of the predictive model, even though they are widely treated as explanations. This perspective aligns with formal frameworks that model explanations as abstract objects without imposing epistemic constraints \citep{karimi2023relationship,hullman2025explanations}. Our goal is not to characterise or evaluate explanations based on their epistemic usefulness, but rather to analyse how different forms of information communicated by the DM influence agents’ behaviour. We refer to \Cref{apx:epistemic-vs-strategic} for a discussion of the distinction between epistemic and strategic roles of explanations.

We provide some specific examples for global and local explanations that fit into our setting (\Cref{def:explanation-method}):

\begin{itemize}
    \item Global surrogate models such as linear models or decision trees that approximate the DM's predictive model $g$ \citep{molnar2020interpretable}. In this scenario, a \textit{constant} explanation $e\in\mathcal{E}$, regardless of $\basewt{x}$, is some surrogate model $f:\mathcal{X}\to\mathcal{Y}$ for the true model $g:\mathcal{X}\to\mathcal{Y}$. The explanation space $\mathcal{E}$ is a subset of linear models or decision tree models: $\mathcal{F}=\{f^\prime:\mathcal{X}\to\mathcal{Y}\}$.
    
    \item Partial descriptions of the DM's predictive model \citep{cohen2024bayesian}. In this scenario, a \textit{constant} explanation $e\in\mathcal{E}$, regardless of $\basewt{x}$, is a \textit{partial} description (of $g$) that corresponds to some subset of the hypothesis space $\mathcal{G}_S\subseteq\mathcal{G}$ such that $g\in\mathcal{G}_S$. This partial description narrows down the agents’ uncertainty about $g$ without fully revealing $g$.
    
    \item Feature attribution-based explanation methods that assign importance scores to features such as SHAP \citep{lundberg2017unified}. For example, when the covariate $\basewt{x}\in\mathcal{X}\subseteq\mathbb{R}^d$ is a vector of $d$ features, an explanation $e_t=(e_{t1}, \ldots, e_{td})\in\mathcal{E}\subseteq\mathbb{R}^d$ is a vector containing the importance scores of each features in $\basewt{x}$.
    
    \item Local surrogate models such as Taylor expansions \citep{xie2024non}. An explanation $e_t$ is some function $f_t:\mathcal{X}\to\mathcal{Y}$ that approximates $g$ in a local neighbourhood of $\basewt{x}$ and the explanation space $\mathcal{E}$ is a subset of all such functions, e.g., $\mathcal{E}\subseteq\mathcal{F}=\{f^\prime\mid f:\mathcal{X}\to\mathcal{Y}\}$.
    
    \item Counterfactual explanations \citep{wachter2017counterfactual}. In this scenario, $e_t=(\rec{x}, g(\rec{x}))$ and $\mathcal{E}=\{(\rec{x}, g(\rec{x})):g(\rec{x})<\hat{\base{y}}\}$, where $\rec{x}$ denotes the recommended covariate for the agent to change to, in order to receive a more favourable prediction, i.e., $g(\rec{x})$ from the DM, e.g., lower insurance premium: $g(\rec{x})<\hat{\base{y}}$.
\end{itemize}

\subsection{A Shapley value example}\label{apx:shapley}
We provide a simple example using Shapley values to illustrate the point made in \Cref{sec:agent-responses}: many attribution methods fail to offer clear, actionable guidance for strategic agents. This limitation arises because Shapley values depend on the underlying data distribution and therefore may not reliably capture the behavior of the DM's predictive model $g$.

\begin{example}\label{example:shapley}
Consider the predictive model $g(x)=x_1-x_2^2$ for any $x=[x_1\quad x_2]^\top\in\mathbb{R}^2$, and an agent with the base feature vector $\base{x}=[16\quad 4]^\top$. With a single agent, we drop the subscript $t$ for simplicity. Furthermore, suppose that the 2 features $\base{X}_1, \base{X}_2$ are statistically independent and that $\base{X}_2\sim\mathcal{U}([2,5])$. The Shapley value of this agent's 2nd feature is
\begin{align*}
&\phi_2(g, \base{x}) 
\\
&= \frac{1}{2}\Big(g(\base{x}) - \mathbb{E}\left[g(\base{X}_1=16, \base{X}_2)\right]
\\
&\hspace{12mm} +\mathbb{E}\left[g(\base{X}_1, \base{X}_2=4)\right] - \mathbb{E}\left[g(\base{X}_1, \base{X}_2)\right]\Big)
\\
&= \frac{1}{2}\Big(0 - \mathbb{E}\left[16-\base{X}_2^2\right] + \mathbb{E}\left[\base{X}_1-16\right] - \mathbb{E}\left[\base{X}_1-\base{X}_2^2\right]\Big) 
\\
&= \frac{1}{2}\Big(-32+2\mathbb{E}\left[\base{X}_2^2\right]\Big) = -3 < 0
.
\end{align*}

However, if we consider $\base{X}_2\sim\mathcal{U}([2,8])$, then $\mathbb{E}\left[\base{X}_2^2\right]=28$ and the Shapley value for this 2nd feature is
\begin{align*}
\phi_2(g, \base{x}) = \frac{1}{2}\Big(-32+ 2\mathbb{E}\left[\base{X}_2^2\right]\Big)=12>0
.
\end{align*}

On the other hand, the partial derivative of the function $g$ at the point $x_2=\base{x}_2=4$ is
\begin{align*}
\frac{dg}{dx_2}(x_2=4) = -8,
\end{align*}
which implies that this agent can achieve a lower prediction score by increasing the value of their 2nd feature $\base{x}_2$. However, the sign of the Shapley value for this feature, as we have shown, can vary depending on the distribution of the data, as a result, this Shapley value cannot say how this agent should change their feature to obtain better prediction score.
\end{example}

\subsection{A misled agent}\label{apx:misled-agent-example}
We present a simple example showing how linear surrogate models cannot guarantee that the induced agent's responses are non-harmful (\Cref{def:no-harm}). This is because they do not satisfy the necessary condition in \Cref{theorem:surrogate-necessary}.

\begin{example}
Suppose that an insurance company uses the $g(x)=x^2$ to predict the risk of a customer whose feature takes on the base value $\base{x}=5$, which corresponds to the prediction $\hat{\base{y}}=25$. With this single-agent scenario, we drop the subscript $t$ for simplicity. Suppose further that this customer has the following cost function for changing their feature:
\begin{align*}
c(\base{x},x) := c(\Delta x) := 
\left\{\begin{aligned}
&3(\Delta x)^2 \quad \forall \Delta x\in (-\infty,-3]\\
&9|\Delta x| \hspace{6mm} \forall \Delta x\in[-3,0]\\
&>0 \hspace{8mm} \forall \Delta x\in(0,\infty),
\end{aligned}\right.
\end{align*}
where we use $\Delta x$ to denote $x - \base{x}$ for any $x\in\mathcal{X}$, and rewrite the cost function into $c(\Delta x)$ for simplicity.
The DM discloses a localised linear model $f(x)=10x-25$ tangent to $g(x)$ at the base value $\base{x}=5$. As the agent wants to minimise their predictive risk, their true utility function and surrogate utility function, to be maximised are respectively
\begin{align*}
u(g,x) &= -g(\base{x}+\Delta x) - c(\Delta x) 
= -(5+\Delta x)^2 - c(\Delta x),
\\
u(f,x) &= -f(\base{x}+\Delta x) - c(\Delta x)
= -10(5+\Delta x)+25 - c(\Delta x).
\end{align*}
% $u(g,x) = -(g(\base{x}+\Delta x)+c(\Delta x)) 
% = -((5+\Delta x)^2 + c(\Delta x))$
% and
% $u(f,x) = -(f(\base{x}+\Delta x)+c(\Delta x))
% = -(10(5+\Delta x)-25 + c(\Delta x))$, respectively.
Hence, $x=2$ (or equivalently $\Delta x=-3$) is the customer's best response as it maximises their surrogate utility function $u(f,x)$. However, this leads to a reduction in the true utility function, i.e., $u(g,2) < u(g,5)$, because the customer has paid a high cost only to achieve a small reduction in their prediction value.
\end{example}

\subsection{On diversity of agents' responses}\label{apx:diversity-of-responses}
We give an example on how \ARexes can fail to capture the diversity of the agents' responses that can be induced by other explanation methods. This helps illustrate why it is important for the DM to acquire additional information about agents as implied by \Cref{assumption:subhomo-response}.

\begin{example}
Suppose that there are three agents who have the same base covariates $\base{x}$ and that their cost factors are respectively $\alpha_1, \alpha_2, \alpha_3$.
Their cost functions have the form $\alpha_t\|x-\basewt{x}\|_2^2$.
From \Cref{def:explanation-method}, because all these three agents have the same base covariates $\base{x}$, they will receive the same explanation, regardless of the explanation method.

Suppose that these agents receive the same surrogate function $f:\mathcal{X}\to\mathcal{Y}$ as an explanation. Then, from \Cref{subsec:surrogate-models}, these agents responses are
\begin{align*}
x_t := \arg\max_x\{-f(x)-\alpha_t\|x-\base{x}\|_2^2\}.
\end{align*}
Because the agents' utility functions only differ in $\alpha_t$, they can result in different maximisers $x_t$, depending on the choice for surrogate function $f$. Thus, this surrogate function induces maximally three actions from these agents.

In contrast, if all these agents receive an \ARex $(\rec{x}_\bullet,\hat{\rec{y}}_\bullet)$, at most two distinct actions can be induced, as implied by \Cref{eq:ar-agent-model}.
\end{example}

\subsection{On \Cref{assumption:subhomo-response}}\label{apx:more-on-conditional-homogeneity}

\Cref{assumption:subhomo-response} can be viewed as requiring the DM to have enough information about agents (e.g., cost-related information $c_t$ being captured in $\basewt{x}$) so that they can partition agents into subgroups with homogeneous responses. It is important to note that this does not assume away the heterogeneity of agents, even within each subgroup.

The following example illustrates the distinction between \Cref{assumption:subhomo-response} and the full-information premise in information design, as well as the distinction between our sufficiency result and the concept of obedience in prior work. This also highlight the significance of our assumption as a relaxation of the full-information premise.

\begin{example}
Suppose that there is a population of heterogeneous agents (e.g., loan applicants) each of which has observable features $x$ (e.g., income, employment status, credit history) and a hidden continuous feature $z$ (e.g., capturing their financial risk preferences or cost sensitivity). Suppose an agent's response is determined by $\psi(x,\delta(z),e)$ where $e$ is an explanation and $\delta(z)$ bins $z$ into behavioural categories(e.g., risk-averse, risk-neutral, risk-seeking). Suppose there exists a deterministic mapping\footnote{Note that this does not necessarily mean that one of them is the cause of the other.} from $x$ to $\delta(z)$, reflecting the fact that applicants with similar observable features tend to fall into the same behavioral category.

In our work, if the DM observes $x$ and conditions on this, then agents are split into subgroups of homogeneous responses, because of the existence of said deterministic mapping. Thus, \Cref{assumption:subhomo-response} holds. However, in information design where obedience is studied, the DM would need to additionally know $\psi$ and $\delta(z)$ to know exactly how an agent responds. This shows the significance of our relaxation: it does not require the DM to know more about agents. Moreover, agents’ heterogeneity remains, because within each group, $z$ varies and across the population, both $z$ and $\delta(z))$ vary.
\end{example}

This relaxation has practical implications. Since the DM only needs partial information to construct behaviorally homogeneous groups, lightweight elicitation methods (e.g., surveys or recorded past behavior) may suffice. This makes \Cref{assumption:subhomo-response} more applicable than full-information premises.

\section{Additional Results and Extensions}\label{apx:additional-results}

\subsection{On the epistemic and strategic roles of explanations}\label{apx:epistemic-vs-strategic}
Explanations often serve two distinct roles: an epistemic role (helping users understand the model) and a strategic role (guiding users toward beneficial actions). A good explanation should therefore satisfy both roles. However, an explanation may be epistemically sound (i.e., faithful to the underlying prediction model) yet still fail strategically if it leads users to take inappropriate actions. More details and illustrative examples can be found in \Cref{theorem:surrogate-necessary} and \Cref{apx:shapley,apx:misled-agent-example}. When this occurs (e.g., when an explanation inadvertently encourages a costly but unhelpful change), users may experience worse outcomes and view the system as unreliable or unhelpful, ultimately eroding trust in the DM’s system. As noted in the Introduction, such effects matter to the decision maker (DM) because user trust is tied to long-term objectives, such as maintaining customer retention or reducing disputes.

\subsection{A necessary condition in broader settings}\label{apx:extended-necessary-condition}
The result in \Cref{theorem:surrogate-necessary} also extends to broader agent models where agents form beliefs about $g$ based on explanations, a setup that is considered in the concurrent work by \citet{cohen2024bayesian}.

% \Cref{theorem:surrogate-necessary} also applies to more general agents' reaction models where agents interpret explanations as some surrogate functions such as in feature-additive methods~\citep{luexplainable} or where agents have some beliefs about the unknown function $g$ ~\citep{cohen2024bayesian}.

\begin{corollary}[Necessary condition]\label{corollary:general-necessary}
For any agent with a tuple $(\basewt{x}, c_t, z_t)$, suppose that 
\begin{enumerate}[label=(\arabic*)]
    \item $\theta_t\in\Theta\subset\mathbb{R}^d$ is a random variable distributed according to the agent's prior $p(\theta_t)$ over the unknown parameter $\theta_0$ of the true predictive model $g_{\theta_0}$,
    \item $p(\theta_t|e_t)\propto p(e_t|\theta_t)p(\theta_t)$ is the posterior belief of this agent upon receiving the explanation $e_t$,
    \item $f_t$ represents the agent's updated belief about $g_{\theta_0}$ defined as $f_t(x) := \mathbb{E}_{\theta_t}[g_{\theta_t}(x)\ \lvert\ e_t],  \forall x\in\mathcal{X}$.
\end{enumerate}
Then, the result in \Cref{theorem:surrogate-necessary} extends to this setting.
\end{corollary}

For Bayesian agents, this corollary shows that the safeguard of \Cref{theorem:surrogate-necessary}, i.e., \Cref{eq:surrogate-necessary-cond}, remains the necessary requirement to guarantee non-harmful responses. It thus guides the DM in designing explanations that prevent agents' self-harming actions, while also serving as a benchmark for whether a no-harm guarantee is even achievable. Although this safeguard may be stringent, this reflects its role as a necessary requirement: \emph{if it fails, no explanation policy can ensure that all agents avoid harmful actions}.

\subsection{On the behavioural model in response to \ARex, in \Cref{eq:ar-agent-model}}\label{apx:more-on-choice-model}
The behavioural model in \Cref{eq:ar-agent-model} which we adopt from prior work \citep{tsirtsis2020decisions} assumes that the agent $t$ chooses to either keep their base covariate $\basewt{x}$ or follow the recommendation $\rec{x}_t$, rather than exploring some other action $x^\prime$. Such behaviour implies that the agent lacks side information about the DM's predictive model $g$ and this discourages the agent from picking another covariate update $x^\prime$ whose benefit is unknown. This is, in fact, compatible with the setting we focus on: \textit{\textbf{agents rely solely on the explanation provided by the decision maker}}. Our goal is then to study how different explanation types, when they are the only source of information, influence agent behavior. In contrast, when agents have access to additional side information that strongly influences their decisions, it becomes difficult to disentangle the behavioural effects of different explanation types. Such settings, where explanations interact with richer sources of information, are an interesting direction for future work.

In practice, there are cases where agents lack side information because repeated exploration might be costly, risky, or restricted. For instance, in insurance pricing repeated probing may be limited to prevent model leakage or may require expensive external data checks. In loan pricing, each application may trigger hard credit checks and fees, discouraging exploration.

Nevertheless, we briefly discuss some alternative models and their implications below.

\textbf{Multiple recommendations as an \ARex.} Consider the case where the DM releases an \ARex that contains $k$ recommended covariate updates and their respective prediction scores, i.e., $e=\{(\rec{x},\hat{\rec{y}})_j\}_{j=1}^k$. Assuming that the agent either picks a covariate update out of this recommended set or does nothing, we then define the set of feasible covariate updates for this agent as $\mathcal{X}_t^\text{fe} = \{\basewt{x}\}\cup\{\rec{x}_j\}_{j=1}^k$. Formally, the behavioural model in \Cref{eq:ar-agent-model} becomes:
\begin{align*}
x_t := \arg\max_{x\in\mathcal{X}_t^\text{fe}}u_t(g,x).
\end{align*}

Then, the results in \Cref{remark:ar-noharm,proposition:ar-sufficiency} also hold, under some technical assumptions on the predictive model $g$. In addition, when \Cref{assumption:subhomo-response} holds, \Cref{theorem:ar-method-sufficiency} also holds.

\textbf{Action exploration with some probability.} Consider the following agents' reaction model under \ARexes:
\begin{align*}
x_t := \rec{x}_t &\text{ if } u_t(g,\rec{x}_t) > u_t(g,\basewt{x})
\\
&\text{ else } 
\left\{
    \begin{aligned}
    &\basewt{x} \text{ with probability } p,\\
    &x^\prime \text{ with probability } 1-p,
    \end{aligned}
\right.
\end{align*}
where the choice of an agent's exploration $x^\prime$ depends on both the explanation $e_t$ and any side information that is available to the agent. The agent’s choice may reflect prior experience, social influence, or individual behavioural traits (e.g., risk aversion or risk seeking). 

Deriving formal guarantees on the agent’s behaviour in such settings typically requires strong assumptions about the decision maker’s knowledge of the agent’s behavioural model, as in information design \citep{bergemann2019information}. By contrast, our formulation adopts a simpler behavioural model that avoids such assumptions: we only require the decision maker to partition agents correctly, as we also discuss in \Cref{apx:more-on-conditional-homogeneity}.

\section{Proofs of the Main Results}\label{apx:proofs}
This section contains the derivations and proofs of our main results presented in \Cref{sec:agent-responses} and \Cref{sec:experiments}.

\subsection{DM's objective}\label{subapx:expanded-obj}
We expand the DM's original objective in \Cref{sec:experiments} to show how all the parameters affect the objective:
\begin{align*}
&\min_{g,\sigma}\mathbb{E}_{P_{X,Z}}\left[\big(g(X)-h(X,Z)\big)^2\right]
\\
=&\min_{g,\sigma}\mathbb{E}_{P_{\base{X},C,Z}}\left[\Big(g\big(\underbrace{\psi(\base{X},\sigma(\base{X},g),Z,C)}_{X}\big)-h\big(\underbrace{\psi(\base{X},\sigma(\base{X},g),Z,C)}_{X},Z\big)\Big)^2\right],
\end{align*}
where the random variable $X:=\psi(\base{X},\sigma(\base{X},g),Z,C)$ denotes the response of an agent and $\psi$ is the response function, as defined in \Cref{sec:problem-formulation}.

% \subsection{Proof for \Cref{prop:localisbetter}}
% \begin{proof}
% For any pair of $\big(e,\basewt{x}\big)\in\mathcal{E}\times\mathcal{X}$, we have
% \begin{align}
% \mathbb{E}_{Z_t}\left[l\big(\basewt{x},e,Z_t\big)\ \big\lvert\ \basewt{x}\right] 
% \geq 
% \min_{e_t\in\mathcal{E}} \mathbb{E}_{Z_t}\left[l\big(\basewt{x},e_t,Z_t\big)\ \big\lvert\ \basewt{x}\right]
% .
% \end{align}

% Because $l\geq 0$, we have the following for any $e\in\mathcal{E}$:
% \begin{align}
% \mathbb{E}_{\basewt{X},Z_t}\left[l\big(\basewt{X},e,Z_t\big)\right] 
% \geq 
% \mathbb{E}_{\basewt{X}}\left[\min_{e_t\in\mathcal{E}} \mathbb{E}_{Z_t}\left[l\big(\basewt{X},e_t,Z_t\big)\ \big\lvert\ \basewt{X}\right]\right]
% \end{align}

% Therefore,
% \begin{align}
% \min_{e\in\mathcal{E}}\mathbb{E}_{\basewt{X},Z_t}\left[l\big(\basewt{X},e, Z_t\big)\right] 
% \geq 
% \mathbb{E}_{\basewt{X}}\left[\min_{e_t\in\mathcal{E}}\mathbb{E}_{Z_t}\left[l\big(\basewt{X},e_t, Z_t\big)\ \big\lvert\ \basewt{X}\right]\right].
% \end{align}
% \end{proof}

\subsection{Proof for \Cref{theorem:surrogate-necessary}}
We introduce the following lemma to help proving \Cref{theorem:surrogate-necessary}.
\begin{lemma}
If the necessary condition (i.e., \Cref{eq:surrogate-necessary-cond}) is violated, then there exist a value $x_\bullet \in \mathcal{X}_t^{g\downarrow}$ and a cost function $c_t$ satisfying the following three conditions:
\begin{align}\label{eq:violation-of-necessity}
\left\{\begin{aligned}
&0 < g(\basewt{x}) - g(x_\bullet) < c_t(\basewt{x}, x_\bullet),
\\
&c_t(\basewt{x}, x_\bullet) < f_t(\basewt{x}) - f_t(x_\bullet),
\\
&c_t(\basewt{x},x_\bullet) + \Big(f_t(x_\bullet)-f_t(x)\Big) < c_t(\basewt{x},x)\quad \forall x \in\mathcal{X}_t^{g\downarrow}\setminus\{x_\bullet\}.
\end{aligned}\right.
\end{align}
\end{lemma}

\begin{proof}

The first line of \Cref{eq:violation-of-necessity} says that the change in the prediction value is smaller than the cost for updating the agent's covariate. We have $0<g(\basewt{x})-g(x_\bullet)$ because of the definition of $\mathcal{X}_t^{g\downarrow}$, and $c_t(\basewt{x},x_\bullet)>0$ because of \Cref{def:cost-function}. Because these two definitions are unrelated, there exist infinitely many such pairs of $\{x_\bullet,c_t\}$.

The second line of \Cref{eq:violation-of-necessity} says that the cost for updating the agent's covariate is smaller than the change in the \textit{surrogate prediction} value (i.e., $f_t(\cdot)$). Because \Cref{eq:surrogate-necessary-cond} is violated, there exists $x_\bullet \in \mathcal{X}_t^{g\downarrow}$ such that $g(\basewt{x})-g(x_\bullet)<f_t(\basewt{x})-f_t(x_\bullet)$. Given such $x_\bullet$, there exists a cost function $c_t$ such that $g(\basewt{x})-g(x_\bullet) < c_t(\basewt{x},x_\bullet) < f_t(\basewt{x})-f_t(x_\bullet)$. This is because the cost function $c_t$ can be designed independent of $\{g,f_t,\basewt{x},x_\bullet\}$.

Before explaining the meaning of the third line of \Cref{eq:violation-of-necessity}, we show how a pair $\{x_\bullet,c_t\}$ can satisfy this condition. Given $\{x_\bullet, f_t\}$, one can design a cost function $c_t$ that satisfies the following, without violating the first two conditions of \Cref{eq:violation-of-necessity}:
\begin{align*}
\left\{\begin{aligned}
c_t(\basewt{x},\basewt{x}) &:= 0,
\\
c_t(\basewt{x},x) &:= c_t(\basewt{x},x_\bullet) + \Big(f_t(x_\bullet)-f_t(x)\Big) + \Big|f_t(x_\bullet)-f_t(x)\Big|
\quad 
\forall x \in\mathcal{X}_t^{g\downarrow}\setminus\{x_\bullet\},
\end{aligned}\right.
\end{align*}
where $c_t(\basewt{x},x_\bullet)>0$ and $\Big(f_t(x_\bullet)-f_t(x)\Big) + \Big|f_t(x_\bullet)-f_t(x)\Big| \geq 0$ for all $x \in\mathcal{X}_t^{g\downarrow}\setminus\{x_\bullet\}$. This holds because we impose no additional restrictions on $c_t$—such as smoothness—beyond those specified in \Cref{def:cost-function}. As a result, the design for $c_t(\basewt{x},x)$, for all $x\neq x_\bullet$, is not affected by the design for $c_t(\basewt{x},x_\bullet)$.

This concludes the proof for this lemma. We explain the third condition in \Cref{eq:violation-of-necessity} in the next proof.
\end{proof}

We now prove \Cref{theorem:surrogate-necessary}.

\begin{proof}[Proof of \Cref{theorem:surrogate-necessary}]
We prove this by contrapositive. Suppose that the necessary condition (\Cref{eq:surrogate-necessary-cond}) is violated, then there exist a value $x_\bullet \in \mathcal{X}_t^{g\downarrow}$ and a cost function $c_t$ satisfying \Cref{eq:violation-of-necessity}.

We explain the meaning of the third condition in \Cref{eq:violation-of-necessity}. Observe that this condition is equivalent to
\begin{align*}
\underbrace{-f_t(x_\bullet) - c_t(\basewt{x},x_\bullet)}_{u_t(f_t,x_\bullet)} > \underbrace{-f_t(x) - c_t(\basewt{x},x)}_{u_t(f_t,x)}\quad \forall x \in\mathcal{X}_t^{g\downarrow}\setminus\{x_\bullet\}.
\end{align*}

Then, using \Cref{eq:violation-of-necessity}, we can see that, if the best response of this agent $t$, against the local surrogate function $f_t$, lies in the set $\mathcal{X}_{t}^{g\downarrow}\cup\{\basewt{x}\}$, then $x_\bullet$ is the solution, since
\begin{align}
x_\bullet &:= \arg\min_{x\in\mathcal{X}_{t}^{g^\downarrow}\cup\{\basewt{x}\}} f_t(x) + c_t(\basewt{x},x)
\\
&:= \arg\max_{x\in\mathcal{X}_{t}^{g^\downarrow}\cup\{\basewt{x}\}} u_t(f_t,x)
.
\end{align}

Moreover, because $g(\basewt{x}) - g(x_\bullet) < c_t(\basewt{x},x_\bullet)$, as specified in the first condition of \Cref{eq:violation-of-necessity}, we have
\begin{align}
-\Big(g(x_\bullet) + c_t(\basewt{x},x_\bullet)\Big) &< -g(\basewt{x})
\\
\Rightarrow u_t(g,x_\bullet) &< u_t(g,\basewt{x}).
\end{align}
This means that $x_\bullet\not\in\nu_t$ (\Cref{def:no-harm}).

On the other hand, if the best response is some $x_\square\not\in(\mathcal{X}_t^{g^{-}}\cup\{\basewt{x}\})$, we have
\begin{align}
g(x_\square) + \underbrace{c_t(\basewt{x},x_\square)}_{>0} > g(x_\square) \geq g(\basewt{x})
,
\end{align}

which results in $u_t(g, x_\square) < u_t(g,\basewt{x})$. 

In either of both cases ($x_\bullet$ or $x_\square$), the agent's response does not belong to the non-harmful set $\nu_t$. Hence, by the contrapositive, ensuring that the agent's responses lie in $\nu_t$ requires that \Cref{eq:surrogate-necessary-cond} holds. This concludes the proof.
\end{proof}

\subsection{A sufficient condition to guarantee non-harmful responses}

\begin{theorem}[Sufficient condition]
Given a base covariate value $\basewt{x}\in\mathcal{X}$ and a surrogate model $f_t:\mathcal{X}\to\mathcal{Y}$, if it holds that
\begin{align*}
f_t\big(\basewt{x}\big) - f_t(x) \leq g\big(\basewt{x}\big) - g(x) \quad \forall x \in \mathcal{X},
\end{align*}
then, for any agent with the same base covariate $\basewt{x}$, their response $x_\bullet$, against the surrogate utility $u_t(f_t,\cdot)$, lies within their non-harmful set $\nu_t$.
\end{theorem}

\begin{proof}
For any arbitrary agent $t$ with the base covariate $\basewt{x}$ and the cost function $c_t$, suppose we have
\begin{align}\label{apx-eq:surrogate-sufficient-condition}
f_t\big(\basewt{x}\big) - f_t(x) \leq g\big(\basewt{x}\big) - g(x) \quad \forall x \in \mathcal{X}.
\end{align}

Let $x_\diamond\in\mathcal{X}$ denotes the agent's best response against the surrogate utility $u_t(f_t,\cdot)$, then
\begin{eqnarray*}
&u_t(f_t,x_\diamond) &\geq u_t(f_t,\basewt{x})
\\
&-f_t(x_\diamond) - c_t\big(\basewt{x}, x_\diamond\big) &\geq -f_t\big(\basewt{x}\big) -\underbrace{c_t(\basewt{x},\basewt{x})}_{=0}
\\
\Rightarrow &f_t\big(\basewt{x}\big) - f_t(x_\diamond) &\geq c_t\big(\basewt{x}, x_\diamond\big).
\end{eqnarray*}

Using the assumed condition in \Cref{apx-eq:surrogate-sufficient-condition}, we obtain:
\begin{eqnarray*}
&c_t\big(\basewt{x}, x_\diamond\big) &\leq f_t\big(\basewt{x}\big) - f_t(x_\diamond)
\leq g\big(\basewt{x}\big) - g(x_\diamond)
\\
\Rightarrow &c_t\big(\basewt{x}, x_\diamond\big) &\leq g\big(\basewt{x}\big) - g(x_\diamond)
\\
\Rightarrow &-g\big(\basewt{x}\big) -\underbrace{c_t(\basewt{x},\basewt{x})}_{=0} &\leq -g(x_\diamond) -c_t\big(\basewt{x}, x_\diamond\big)
\\
\Rightarrow &u_t(g,\basewt{x}) &\leq u_t(g,x_\diamond).
\end{eqnarray*}

Then $x_\diamond\in\nu_t$. This concludes the proof.
\end{proof}

\subsection{Proof for \Cref{theorem:ar-method-sufficiency}}
We introduce the following lemma to help proving the theorem.
\begin{lemma}\label{lemma:for-ar-method-sufficiency}
Consider a subset of $T^\prime$ agents who have the same base covariate value, i.e., $\basewt{x}=\base{x}$ for all $t\in[T^\prime]$, for some value $\base{x}\in\mathcal{X}$. We assume that agents respond to \ARexes according to \Cref{eq:ar-agent-model} and that \Cref{assumption:subhomo-response} holds for all explanation methods. Then, all $T^\prime$ agents have the same set of non-harmful responses, i.e., $\nu_t=\nu$ for all $t\in[T^\prime]$. Moreover, for any arbitrary explanation method $(\mathcal{E}, \sigma)$ that induce a non-harmful response $x_\bullet$ from all of these agents, there exists an \ARex $(\rec{x},\hat{\rec{y}})$ that induces the same response $x_\bullet$ from these agents.
\end{lemma}

\begin{proof}
Because these agents' responses are always identical for any explanation (\Cref{assumption:subhomo-response}), they must have the same response for any AR-based explanation. Given that AR-based explanations always induce non-harmful responses (\Cref{remark:ar-noharm}), these agents will have the same set of non-harmful responses, i.e., $\nu_t=\nu$ for all $t\in[T^\prime]$. This proves the first result.

Furthermore, for any arbitrary explanation method that induces a best response $x_\bullet\in\nu$ of these agents, it must hold that $u_t(g,x_\bullet)\geq u_t(g,\base{x})$ for all $t\in[T^\prime]$, because of \Cref{def:no-harm}.

Consequently, there exists an AR-based explanation $(\rec{x},\hat{\rec{y}})=(x_\bullet,g(x_\bullet))$ and by \Cref{eq:ar-agent-model}, all agents will follow the recommendation. 

This concludes the proof.
\end{proof}

We now prove the theorem.
\begin{proof}[Proof for \Cref{theorem:ar-method-sufficiency}]
Fix an \ARex space $\mathcal{E}^\prime=\mathcal{X}\times\mathcal{Y}$ as defined in \Cref{subsec:action-rec} and let us define $\sigma_{\base{x}}^\prime: g\mapsto \rec{x}$ as an action-recommendation function, parameterised by $\base{x}$. It can be seen that, for an agent population with the corresponding set of base covariates $\base{\mathcal{X}}_T=\{\basewt{x}\}_{t\in[T]}$, any \ARex policy $\sigma^\prime:(\base{x},g)\mapsto(\rec{x},\hat{\rec{y}})$ can be constructed with $\{\{\sigma^r_{\base{x}}\}_{\base{x}\in\base{\mathcal{X}}_T},g\}$.

Suppose that we have a specific predictive model $g$ and an arbitrary explanation method $(\mathcal{E}, \sigma)$ that induces non-harmful responses for all $T$ agents in the population.

Consider any subset of $T^\prime$ agents who have the same base covariates, i.e., $\basewt{x}=\base{x}$ for all $t\in[T^\prime]$, for some value $\base{x}$. Suppose that the said explanation method $(\mathcal{E}, \sigma)$ induces the non-harmful response $x_\bullet$ for these agents. Then, following \Cref{lemma:for-ar-method-sufficiency}, there exists an \ARex $(\rec{x}, \hat{\rec{y}})$ that induces the same response $x_\bullet$. Consequently, we can construct the corresponding action-recommendation function as $\sigma^\prime_{\base{x}}:=\rec{x}$.

Putting together all those action-recommendation functions, we can construct an \ARex policy that induce the same response tuple from the whole agent population as that induced by the said arbitrary explanation method.

This concludes the proof.
\end{proof}

\section{Additional Details of the Experiments}\label{apx:detailed-exp}
We provide here additional details to the setups for experiments in \Cref{sec:experiments}.

\subsection{On the no-harm guarantee of \ARexes}
We use a quartic function as the predictive model of the DM where $g(x)=x^4-x^2+1$ and use 2nd-order Taylor expansions and LIME as baseline explanation methods.
% To illustrate that Taylor expansions \citep{xie2024non} as explanations (because this explanation method has a clear reaction model, while others do not) do not guarantee no-harm, 
We generate a simple data set of 100 agents with 1-dimensional features as follows: 
\begin{align*}
\basewt{X} &\sim \mathcal{N}(0,0.4^2),\\
\alpha_t &\sim \mathcal{U}([1,1.2]),
\end{align*}
where $\alpha_t$ denotes the cost factor of agent $t$, which we use to model the heterogeneity of agents' cost functions.
The cost function for agent $t$ is $c_t(\basewt{x},x)=|\alpha_t|\|\basewt{x}-x\|_2^2$. For simplicity, we generate the AR-based explanations randomly as follows:
\begin{align*}
\rec{X}_t &\sim \mathcal{N}(0,0.4), \\
\hat{\rec{Y}}_t &:= g(\rec{X}_t).
\end{align*}

\textbf{Computational details.} This experiment took less than 5 seconds to run on a standard MacBook Pro with an M2 chip and 16GB of RAM.

\subsection{Operationalising \ARexes on synthetic data}\label{apx:operation-on-synthetic-experiment}

\paragraph{Synthetic data generation.} We construct a synthetic dataset containing agents of 3-dimensional (observable) feature vector $\basewt{x}\in\mathbb{R}^3$ and a scalar (unobservable) feature $z_t\in\mathbb{R}$ as follows:
\begin{align*}
Z_t &\sim \mathcal{U}(\{0,1,2,3\}),
\\
\alpha_t\mid z_t &\sim \mathcal{N}(0.02+0.1z_t,\ 0.01^2),
\\
\basewt{X}\mid z_t &\sim \mathcal{N}(\mathbf{m},\ I\times 2),
\end{align*}
where $\mathbf{m}:=[10+z_t,10+z_t,10+z_t]^\top$ and the cost function $c_t(\basewt{x},x)=|\alpha_t|\|\basewt{x}-x\|_2^2$.

\paragraph{Learning agents' responses.} As we mention in \Cref{subsec:exp-synthetic}, training $\sigma^r$ requires the DM's ability to simulate agents' responses. Let $\hat{\psi}: (\basewt{x}, \hat{\base{y}}_t, \rec{x}_t, \hat{\rec{y}}_t)\mapsto \hat{x}_t$ denote a model that the DM can use to predict agents' responses. We construct $\hat{\psi}$ by using an underlying model $\xi:\big(\basewt{x}, \rec{x}_t, \Delta \rec{g}_t\big)\mapsto \hat{w}_t$ that predicts an agent's compliance $w_t$. We define compliance $w_t$ as a binary variable where $w_t=1$ indicates the agent follows the recommendation $\rec{x}_t$, and $w_t=0$ otherwise. The term $\Delta \rec{g}_t:=g(\basewt{x})-g(\rec{x}_t)$ denotes the gain in prediction value for the agent $t$ and serves as a useful feature for this classifier, given the additive structure of utility in \Cref{eq:ag-true-utility}. 
Once $\xi$ is learned, the DM can simulate an agent's response as $\hat{x}_t := \hat{w}_t \rec{x}_t + (1-\hat{w}_t) \basewt{x}$.

To generate necessary data to learn the classifier $\xi$, the DM can employ a sampler $\pi$ to generate random \ARexes as follows: 
\begin{align*}
\rec{X}_t\mid \basewt{x} &\sim \pi(\rec{X}_t, \basewt{x}), \\
\hat{\rec{Y}}_t &:= g(\rec{X}_t).
\end{align*}

\textbf{Repeated risk minimisation.} Putting everything together, with finite samples, the empirical objective in the $i$-th iteration of the RRM procedure is
$(g_{i},\sigma_{i}^r)=\arg\min_{g,\sigma^r}(1/T_i)\sum_{t\in[T_i]}(g(\hat{x}_t)-y_t)^2$. The notation $\hat{x}_t$ refers to the simulated agent's response based on the recommended action $\rec{x}_t:=\sigma^r_i(\basewt{x})$ and the inferred reaction model $\hat{\psi}$. In addition, $\{y_t\}_{t\in[T_i]}$ are the outcomes of $T_i$ agents collected from when the previous models $(g_{i-1},\sigma^{r}_{i-1})$ are deployed, as usually done in RRM~\citep{perdomo2020performative}. 

\paragraph{RRM with baseline CEs}. Because the CE policy is fixed and RRM is already used, there is no need to simulate agents' responses here. Thus, the  objective of this approach in each iteration of RRM is $g_{i}=\min_{g}\sum_{t\in[T_{i}]}(g(x_t)-y_t)^2$ where the dataset $\{x_t,y_t\}_{t\in[T_i]}$ is collected from when the previous model $g_{i-1}$ is deployed.

\paragraph{The \ourprocedure algorithm.} \Cref{algo:main} summarises the details of our \ourprocedure procedure and we explain the steps here. We use 3-layer ReLU network for constructing all three models $g$, $\sigma^r$, and $\xi$. To make the learning more efficient, we first pre-train the predictive model $g$ and the AR function $\sigma^r$ to obtain $g_0$ and $\sigma^r_0$. We use a dataset $D_1=\{x_t,y_t\}_{t\in[5000]}$ for this step.

Then, we interact with the next $10^6$ agents to collect another data set $D_2=\{\basewt{x}, \rec{x}_t, w_t, \Delta \rec{g}_t\}_{t\in[10^6]}$. The collected dataset will later be used to train the compliance predictor $\xi$. To do this, we construct a sampler $\pi$ to generate random \ARexes:
\begin{align*}
\rec{X}_t\mid \basewt{x} &\sim \mathcal{N}(x_t^\diamond,4), \\
\hat{\rec{Y}}_t &:= g_0(\rec{X}_t),
\end{align*}
where $x_t^\diamond$ is chosen arbitrarily between the following options:
\begin{align*}
x_t^\diamond &:= \basewt{x},\quad \text{or}\\
x_t^\diamond &:= \sigma^r_0\big(\basewt{x}\big),\quad \text{or}\\
x_t^\diamond &:= \arg\min_{x\in\mathcal{X}}\left(g_0(x) + \big\|x-\basewt{x}\big\|_2^2\right).
\end{align*}

Next, we run RRM over $m=100$ iterations, in each of which we deploy $g_{i-1}$ and $\sigma^r_{i-1}$ to interact with $10^4$ agents to collect a data set $\{x_t,y_t\}_{t\in[10^4]}$, where the subscript $i$ denotes an iteration, as outlined in \Cref{algo:main}. The models $g_i$ and $\sigma^r_i$ are sequentially updated over $100$ iterations as specified in \Cref{eq:rrm-update} to eventually obtain the optimal $g$ and $\sigma^r$. 
% We perform a similar procedure to obtain $g_\text{glo}$ and a \textit{global} AR mapping, which is in fact just a constant mapping that generate the same recommended action $x^{(r)}_\text{glob}$ for all agents . Then, \Cref{eq:rrm-update} can be replaced with the following objective:
% $(g_{i},x^{(r)}_{i})=\min_{g,x^{r}}\sum_{t\in[T_i]}\left(g(\hat{x}_t)-y_t\right)^2$.

\begin{algorithm}[t!]
\caption{Joint optimisation of $g$ and $\sigma$.}
\label{algo:main}
\textbf{Require:} Dataset $D_1=\{x_t,y_t\}_{t\in[T]}$ and the sampler $\pi$.\\
\textbf{Parameters:} $T, T^\prime$, and $\{T_1,\ldots,T_m\}$.

\begin{algorithmic}[1]
\STATE Pre-train $g$ and $\sigma^r$ on $D_1=\{x_t,y_t\}_{t\in[T]}$ as follows:
\begin{align*}
g_0 &= \arg\min_g \sum_{x_t,y_t\in D_1} (g(x_t)-y_t)^2,
\\
\sigma^r_0 &= \arg\min_{\sigma^r} \sum_{x_t,y_t\in D_1} (\sigma^r(x_t)-x_t)^2.
\end{align*}
\STATE Interact with agents over $T^\prime$ rounds, with $g_0$ and a sampler $\pi$ to collect the dataset $D_2=\{\basewt{x}, \rec{x}_t, \Delta \rec{g}_t, w_t\}_{t\in[T^\prime]}$, then train the compliance predictor with the objective
\begin{align*}
\arg\min_{\xi} \sum_{t\in[T]}\Big(-w_t\log(\hat{w}_t)-(1-w_t)\log(1-\hat{w}_t)\Big),
\end{align*}
where $\hat{w}_t$ is the output of $\xi$ and $w_t$ is the actual label.
\FOR{$i\in\{1,\ldots,m\}$}
\STATE Interact with agents over $T_i$ rounds with $g_{i-1}$ and $\sigma^r_{i-1}$ to collect the dataset $D_{3,i}=\{\basewt{x},y_t\}_{t\in[T_i]}$
\STATE Update $(g_{i},\sigma^r_{i})$ by solving
\begin{align}\label{eq:rrm-update}
\big(g_{i},\sigma^{r}_{i}\big) := \arg\min_{g,\sigma^{r}}\sum_{t\in[T_i]}\left(g(\hat{x}_t)-y_t\right)^2,
\end{align}
where $\hat{x}_t$ is simulated with the compliance predictor $\xi$ and the input $\basewt{x}$.
\ENDFOR
\STATE Set $(g,\sigma^r):=(g_m,\sigma^r_m)$.
\end{algorithmic}
% \vspace{-1em}
\end{algorithm}

% We conduct a similar procedure for the case of counterfactual explanations. Specifically, the objective of this approach in each iteration of RRM is $g_{\text{ce},i}=\min_{g_{\text{ce},i-1}}\sum_{t\in[T_{i}]}(g(\hat{x}_t)-y_t)^2$
% where we do not optimise for the counterfactual explanations generator and simply use the following scheme: $x^\text{ce}_t := \arg\min_{x}(g_\text{ce}(x) + \|x-\basewt{x}\|_2^2)$.

% Finally, we compare the prediction errors between the three approaches (i.e., local AR mapping, global AR mapping, and counterfactual explanations) on a hold-out test set of $10^6$ strategic agents.

\textbf{Evaluation.} We then compare the prediction errors between all approaches on a hold-out test set of $10^6$ strategic agents. For ease of presentation, we scale the mean-squared errors by dividing them with the constant $nc=\frac{1}{T}\sum_{t\in[T]}\basewt{y}$ that is independent of the DM's choice of models. We refer to the scaled errors as normalised mean-squared errors (nMSE). If the loss is computed on offline data, i.e., without agents' strategic responses, we simply refer to it as nMSE, otherwise, \emph{strategic} nMSE.

\paragraph{Hyperparameter choices.} We report all key details necessary to understand the experimental results. Other lower-level settings (e.g., optimiser, learning rates, numbers of iterations, model sizes) are omitted from discussion as they do not affect our main conclusions. Our goal is to demonstrate the benefit of optimising the \ARex policy, rather than relying on fixed designs such as counterfactual explanations. Performing extensive hyperparameter tuning for the baselines would effectively optimise those fixed policies, which would only reinforce our central claim. However, we include complete source code with the supplementary material for reproducing all experimental results.

\textbf{Computational details.} This experiment was completed in approximately 11 minutes on a cloud machine with 44 vCPUs and 88GB of RAM.

\subsection{\ARexes on German credit dataset}
The dataset is publicly available at \url{https://archive.ics.uci.edu/dataset/144/statlog+german+credit+data}.

As mentioned in \Cref{subsec:exp-german-credit}, we adopt details from \citet{xie2024non} for pre-processing the data. Specifically, we remove two sensitive features (i.e., \textit{age} and \textit{sex}) and designate 8 out of the remaining 18 features as modifiable for strategic agents, these are: \textit{existing account status, credit history, credit amount, savings account, present employment, installment rate, guarantors}, and \textit{residence}. Categorical features are label encoded, for simplicity, and numerical features are standardised to have zero mean and unit variance. For each modifiable feature (indexed with $i$), we identify the range of feasible values and denote it as $[x_i^L, x_i^U]$. We use this to prevent the DM from recommending extreme feature changes to agents and to prevent agents from adopting such extreme feature modifications.

As mentioned in \Cref{subsec:exp-german-credit}, we design the cost function as
\begin{align*}
c(\basewt{x},x) := 0.01\sum_{i\in\mathcal{I}} \frac{\left|\base{x}_{ti}-x_{i}\right|}{(x_{i}^U-x_{i}^L)},
\end{align*}
where $\mathcal{I}$ is the index set of modifiable features. Any change in a non-modifiable feature incurs infinite cost. We use the weighted $L1$ distance instead of the quadratic form to avoid the cost values from becoming excessively small. Furthermore, the scaling factor of $0.01$ is heuristically chosen so that we have a nice balance between agents who can easily change their features and agents who cannot. This choice enables us to more clearly observe the impact of different designs for \ARex policies.

In this dataset, the agent's outcome $y_t$ is a binary variable indicating a customer's credit risk classification (i.e., $1$ means \textit{good} and $0$ means \textit{bad}). To simulate how an agent's outcome $y_t$ changes w.r.t. their strategically updated covariate $x_t$, we use a logistic regression model that is fitted on this dataset of 1,000 observations. Let $s:\mathcal{X}\to(0,1)$ be the resulting logistic function that outputs the probability of the outcome $y_t$ is positive. We simulate the agent's outcome as
\begin{align*}
Y_t\mid x_t \sim \text{Bernoulli}\big(s(x_t)\big).
\end{align*}

We fit CTGAN~\citep{ctgan} on the original dataset then generate additional samples: 9,000 for training and 1,000 for testing.

\textbf{Evaluation.} 
% To train $\sigma_{loc}^r$ efficiently, we restrict modifications to actionable features within predefined bounds $[x_i^L, x_i^U]$. 
Counterfactual explanations are generated similarly to the synthetic setup, with additional constraints (e.g., categorical/bounded features) enforced via projected gradient descent. Unlike previous setup, we now use $-g_s(x)$ in the objective function for generating CEs as the agents now benefit from higher prediction scores.
The training of $g$ follows the same procedure as before, except for the loss functions: we use \textit{weighted} binary cross-entropy (BCE) to learn agents' response function $\xi$ (with imbalanced labels), and \textit{unweighted} BCE in RRM, where outcomes $y_t$ may shift due to strategic behavior.
% The training procedure is the same as the previous synthetic setup, except for the loss functions. For the loss functions, we use a \emph{weighted} binary cross-entropy (BCE) in steps 1 and 2, where labels are unbalanced but fixed, and an \emph{unweighted} BCE in RRM, where outcomes $y_t$ may shift due to agents' strategic behavior.
We evaluate predictive accuracy (using the $F_1$ score) on the hold-out test set of 1,000 strategic agents, referring to this as \textit{strategic $F_1$ score} to reflect possible shifts in $(x_t,y_t)$.

\paragraph{Hyperparameter choices.} Similar to the previous synthetic setup, we report all key details necessary to understand the experimental results. Other lower-level settings (e.g., optimiser, learning rates, numbers of iterations, model sizes) are omitted from discussion as they do not affect our main conclusions. We include complete source code with the supplementary material for reproducing all experimental results.

\textbf{Computational details.} This experiment was completed in approximately 4 minutes on a cloud machine with 44 vCPUs and 88GB of RAM.

\subsection{\ARexes on the RSNA Pediatric Bone Age dataset}\label{apx:bone-age-experiment}
We use the 2017 RSNA Pediatric Bone Age dataset \citep{halabi2018rsna} that contains x-ray images of infants' hands and the task is to predict their bone age (in months). The training set consists of $12,611$ images and the test set contains $1,425$ images. The dataset is publicly available at \url{https://www.rsna.org/rsnai/ai-image-challenge/rsna-pediatric-bone-age-challenge-2017}.

\paragraph{Scenario.} We consider the scenario in which a hospital (acting as a DM) wants to acquire more x-ray data to help improve their predictive model $g$, and clinics (acting as strategic agents) want to sell their data.

A clinic $t$ submits an x-ray image $\basewt{x}$ alongside with the recorded patient's outcome $\basewt{y}$. The hospital has the ability to evaluate the quality of this data and to pay the clinic an amount specified by $b(g,\basewt{x},\basewt{y})$.
Let $\kappa:\mathcal{X}\to\mathbb{R}^k$ denotes a function that associate an x-ray image $x$ with a vector of $k$ their properties. For example, this could contain the x-ray machine's settings when the image was taken such as \textit{exposure}, \textit{focus}, or \textit{contrast}.
After evaluating the submitted x-ray image $\basewt{x}$, the hospital could give a recommendation to the clinic on how to obtain higher payment, i.e., $b^\prime$, by re-submitting another image. This recommendation can come from the form of the recommended properties $\kappa^\prime\in\mathbb{R}^k$. The clinic can look for another image $x_t$ such that $\kappa(x_t)=\kappa^\prime$.

After submitting the image $x_t$ (and the associated patient outcome $y_t$), the clinic receives the payment $b(g,x_t,y_t)$. The objective of the hospital is to minimise the expected predictive error of the model $g$, while taking into account the strategic behaviour of the clinics:
\begin{align*}
\min_{g,\sigma}\mathbb{E}\left[(g(X_t)-Y_t)^2\right]
,
\end{align*}
where $\sigma:(g,\base{x})\mapsto(\kappa^\prime,b^\prime)$ denotes the \ARex policy.

\paragraph{Setup details.} Following \citet{lu2024data}, we create the embeddings for the images using a pretrained CLIP ViT-B/32 model. We denote the embedding of an image $x$ as $\phi(x)\in\mathbb{R}^{512}$. To simulate the change in outcome $y$ when the covariate $x$ changes, we assume a linear relationship between these two, in the embedding space, and fit a linear model $h$, where:
\begin{align*}
y := h(x) := \phi(x)^\top\beta + \beta_0.
\end{align*}

Similar to the previous two experiments, we use 3-layer ReLU networks to construct the models $g, \sigma^r, \xi$. For simplicity, we assume here that the properties of an image $\kappa(x)$ coincide with the first $k$ entries in their embedding vector $\phi(x)$. Then, we can define the \ARex policy $\sigma$ and action-recommendation function $\sigma^r$ as follows:
\begin{align*}
\sigma &: (\base{x},g)\mapsto(\kappa^\prime,b^\prime),
\\
\sigma^r &:\base{x}\mapsto\kappa^\prime.
\end{align*}

We set the payment function as $b(g,x,y):=(g(x)-y)^2$ to reward an agent an amount that corresponds to how much the new data point surprises the current predictive model $g$.
For simplicity, we assume that agents operate directly on the embedding space and their cost functions have the form of $c_t(x,\basewt{x}):=\alpha_t\|\phi(x)-\phi(\basewt{x})\|_2^2$, where $\alpha_t$ represents the heterogeneous factor of this cost function.

\paragraph{Baselines and evaluation.} Similar to the previous two experiments, we use fixed designs of counterfactual explanations as baselines and evaluate the performance of $g$ with mean squared errors (MSE).

\paragraph{Computational details.} This experiment was completed in approximately 5 minutes on a machine with an NVIDIA A100 GPU, 12 vCPUs and  85GB of RAM.

%%%
\begin{figure*}[t]
    \centering
    \begin{subfigure}[t]{0.31\textwidth}
        \includegraphics[width=0.98\textwidth]
        {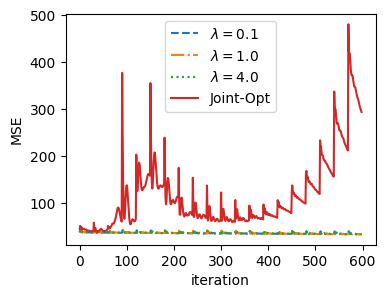}
        \subcaption{From 0th to 600th iteration.}
        % \label{subfig:taylor-harm}
    \end{subfigure}
    \hfill
    \begin{subfigure}[t]{0.31\textwidth}
        \includegraphics[width=0.98\textwidth]
        {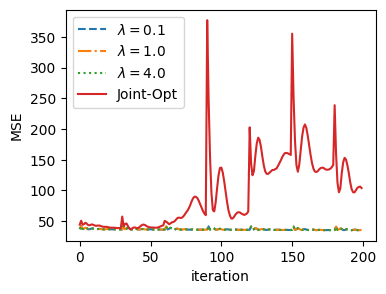}
        \subcaption{From 0th to 200th iteration.}
        % \label{subfig:synthetic-plot}
    \end{subfigure}
    \hfill
    \begin{subfigure}[t]{0.31\textwidth}
        \includegraphics[width=0.98\textwidth]
        {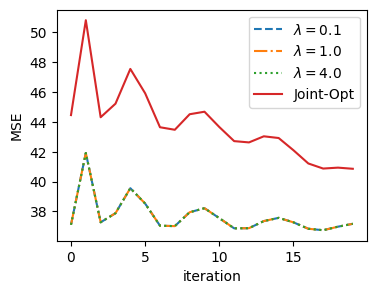}
        \subcaption{From 0th to 20th iterations.}
        % \label{subfig:german-credit-plot}
    \end{subfigure}
    \caption{
        \small 
        The plots show the training errors of our \ourprocedure approach and the baselines at different magnification levels. The training error of \ourprocedure increases \textbf{more} sharply after each interaction between the DM and agents. This implies an unstable optimisation process and the absence of fixed-point convergence \citep{perdomo2020performative}, especially in high-dimensional settings where the number of trainable parameters increases significantly due to joint-training of $\{g,\sigma^r\}$.
    }
    \label{fig:apx-bone-age-result}
\end{figure*}

\paragraph{Results.} \Cref{fig:apx-bone-age-result} shows the training errors of our \ourprocedure approach and the baselines at different magnification levels. The training error of \ourprocedure increases \textbf{more} sharply after each interaction between the DM and agents. This implies an unstable optimisation process and the absence of fixed-point convergence \citep{perdomo2020performative}, especially in high-dimensional settings where the number of trainable parameters increases significantly due to joint-training of $\{g,\sigma^r\}$.

This phenomenon reflects the coupled nature of joint training: as the ARex policy $\sigma$ is updated, it shifts the distribution of agents’ covariates $P_X$, which in turn affects the optimisation of $g$. When the parameter space is large, these feedback effects make convergence difficult.
% This phenonmenon reflects the fact during joint training, as the \ARex policy $\sigma$ is updated, it consequently shifts the distribution of agents' covariates $P_X$. When the parameter space is large, it is hard for the training to converge.
The issue is further compounded by the choice of the DM’s payment function $b$ and the agents’ cost functions $c_t$, both of which strongly influence stability.
% On the other hand, the choice of the DM's payment function and the agents' cost functions plays an important role. 
As pointed out by \citet{perdomo2020performative}, RRM can only converge under some conditions on the loss function.

Consequently, while our joint-optimisation procedure performs reliably in low- to medium-dimensional settings, extending it to high-dimensional and more complex environments remains an open and interesting direction. A systematic study of convergence behaviour in such regimes would complement our current focus.
% Consequently, although we provide a joint-optimisation approach that empirically works in low to medium-dimensional settings, more thorough study is necessary to enable the joint-optimisation approach more applicable in high-dimensional and complex settings.
Related work by \citet{tsirtsis2020decisions} also explores joint-optimisation but restricted to discrete and classification settings.

Overall, our focus lies in establishing a theoretical foundation \textbf{to characterise explanations} in terms of their impact of agents' behaviour. While we prove the existence of \ARexes and \ARex methods for sufficiently inducing no-harm actions in agents, it would be an interesting future direction to thoroughly investigate how one can optimally design an \ARex policy in a high-dimensional setting.
% In conclusion, as our work provides a theory on the existence of \ARexes and \ARex methods for sufficiently inducing no-harm actions in agents, it would be an interesting future direction to thoroughly investigate how one can optimally design an \ARex policy in a high-dimensional setting.

% \newpage
\section{Related Work (Extended)}\label{apx:extended-related-work}

\paragraph{Strategic learning.} Strategic classification was introduced by \citet{bruckner2012static} and further developed by \citet{hardt2016strategic}, where they presented the first computationally efficient algorithms to learn near-optimal classifiers in strategic environments. Their key assumption was that agents have complete knowledge of the classifier due to information leakage, even when the system is designed to obscure the model. In contrast, our work weaken this assumption by considering scenarios where the learner (or DM) releases partial information about their model.

% However, they did not explore scenarios where the learner (or decision maker) releases only partial information about their model.

Subsequent research has expanded the field of strategic classification by developing more efficient algorithms \citep{dong2018strategic, levanon2021strategic, ahmadi2021strategic} or by incorporating new aspects such as social welfare \citep{hu2018disparate, milli2019social}, randomisation \citep{sundaram2023pac, ahmadi2023fundamental, shao2024strategic}, repeated interactions \citep{harris2021stateful, cohen2023sequential}, and incentivising improvements \citep{kleinberg2020classifiers,harris2022strategic,vo2024causal}. 
% Another line of works concerns with providing incentive for strategic agents to improve \citep{kleinberg2020classifiers, shavit2020causal, harris2022strategic, horowitz2023causal, vo2024causal} with many of them focusing on regression settings and incorporating causal reasoning.
Another important thread considers settings where agents cannot best respond to the true model—either due to bounded rationality, limited information, or uncertainty in their response process (e.g., \citep{jagadeesan2021alternative,bechavod2022information,harris2022bayesian,xie2024learning}). Our work is complementary: instead of committing to a particular agent model, we focus on how the DM can structure the information disclosed through explanations, and we conduct a comparative analysis across explanation types to understand when they induce responses that do not harm agents.

In particular, the works of \citet{harris2022bayesian} and \citet{cohen2024bayesian} are closest to ours. As discussed throughout \Cref{subsec:action-rec}, \citet{harris2022bayesian} focuses on the obedience-inducing property (also known as the Bayesian incentive compatibility) of a subclass of action recommendations, whereas we focus on the no-harm property of action recommendations. 
As we also discuss in \Cref{subsec:action-rec}, identifying an AR-based explanation policy that can induce obedience for each \textit{individual} agent is hard, especially when agents are heterogeneous (e.g., in \citet{harris2022strategic,shao2024strategic}), and such \textit{individual}-level identification might not be necessary if the DM only cares about optimising their expected utility, which is computed over the population of agents.
Unlike action recommendations, \citet{cohen2024bayesian} instead releases a subset of the hypothesis class to all agents, aligning with global explanations in our framework (\Cref{sec:problem-formulation}). However, interpreting a set of models—such as neural networks—can be difficult for agents. In contrast, the AR-based explanations are not only more interpretable but also provide guidance that cannot mislead agents.

\paragraph{Counterfactual explanations and algorithmic recourse.} 
\citet{tsirtsis2020decisions,karimi2022survey} explore counterfactual explanations and algorithmic recourse, for strategic agents. Although algorithmic recourse focuses on recommending actions to achieve better outcomes, it actual implementation often requires strong causal assumptions. These assumptions can render it impractical in more general settings where such causal knowledge is not justified.
In contrast, our work adopts a weaker notion of desirability centred on agents' welfare---ensuring non-harmful responses---and examines a broader range of explanation types beyond counterfactuals.
Even though both \citet{tsirtsis2020decisions} and our work involve counterfactual explanations, the contributions differ. Precisely, they focus on optimising CEs in strategic settings, while we analyse multiple explanation types and formally show why \ARexes are more desirable. 
In addition, our proposed learning procedure in \Cref{subsec:exp-synthetic}, though not the main focus, is designed to be more general, extending beyond the classification and discrete case in \citet{tsirtsis2020decisions}.

\paragraph{Information design.} 
The extensive literature on information design, as surveyed by \citet{bergemann2019information}, studies how to design information disclosure policies in a game of two parties. While our results are inspired by these works, e.g., \Cref{theorem:ar-method-sufficiency}, the goals differ significantly. As discussed in \Cref{subsec:action-rec}, information design aims at \textit{persuading} agents with a general response model and does not necessarily ensure the no-harm property (\Cref{def:no-harm}). In contrast, we study explanation methods that prioritise the no-harm property, ensuring agents' welfare is not compromised.
By incorporating specific agent models in strategic settings, we establish the sufficiency of AR-based explanations without requiring the DM to account for agents' heterogeneous reaction models.

\end{document}